\documentclass[openacc]{rsproca_new}

\newtheorem{thm}{Theorem}
\newtheorem{prop}{Proposition}

\newtheorem{cor}{Corollary}
\newtheorem{definition}{Definition}

\newtheorem{principle}{Core Principle}

\usepackage{MnSymbol}
\usepackage[norelsize,boxed,noend,linesnumbered]{algorithm2e}\DontPrintSemicolon
\SetKwFor{RepeatTimes}{repeat}{times}{endrepeat}
\RestyleAlgo{ruled}

\DeclareMathOperator{\bfa}{\mathbf{a}}
\newcommand{\ta}{\tilde{\alpha}}
\newcommand{\tb}{\tilde{\beta}}
\newcommand{\tc}{\tilde{\gamma}}
\newcommand{\meet}{\tilde{\wedge}}
\newcommand{\join}{\tilde{\vee}}

\jname{rspa}
\Journal{Proc R Soc A\ }

\begin{document}

\title{Bits and Pieces: Understanding Information Decomposition from Part-whole Relationships and Formal Logic}

\author{
A. J. Gutknecht $^{1,2}$, M. Wibral$^{1}$ and A. Makkeh$^{1}$}

\address{$^{1}$Campus Institute for Dynamics of Biological Networks, Georg-August University, Goettingen, Germany \\
$^{2}$MEG Unit, Brain Imaging Center, Goethe University, Frankfurt,  Germany}

\subject{Information Theory, Logic, Statistics, Applied Mathematics}


\corres{A. J. Gutknecht\\
\email{agutkne@uni-goettingen.de}}

\begin{abstract}
Partial information decomposition (PID)  seeks to decompose the multivariate mutual information that a set of source variables contains about a target variable into basic pieces, the so called "atoms of information". Each atom describes a distinct way in which the sources may contain information about the target. For instance, some information may be contained uniquely in a particular source, some information may be shared by multiple sources, and some information may only become accessible synergistically if multiple sources are combined. In this paper we show that the entire theory of PID can be derived, firstly, from considerations of part-whole relationships between information atoms and mutual information terms, and secondly, based on a hierarchy of logical constraints describing how a given information atom can be accessed. In this way, the idea of a partial information decomposition is developed on the basis of two of the most elementary relationships in nature: the part-whole relationship and the relation of logical implication. This unifying perspective provides insights into pressing questions in the field such as the possibility of constructing a PID based on concepts other than redundant information in the general n-sources case. Additionally, it admits of a particularly accessible exposition of PID theory. 
\end{abstract}


\maketitle


\noindent
\section{Introduction}
Partial information decomposition (PID) is an example of a rare class of problems where a deceptively simple question has perplexed researchers for many years, leading to heated disputes over possible solutions \cite{Schneidman2003}, simple but incomplete answers \cite{McGill1954}, and even to statements that the question should not be asked \cite{MackayBook}. The core question of PID is how the information carried by multiple source variables about a target variable is distributed over the source variables. In other words, it is the information theoretic question of 'who knows what about the target variable'. Intuitively, answering this question involves finding out which information we could get from multiple variables alike (called redundant or shared information), which information we could get only from specific variables, but not the others (called unique information), and which information we can only obtain when looking at some variables together (called synergistic information).

Examples of questions involving PID, are found in almost all fields of quantitative research. In neuroscience, for instance, we are interested in how the activity of multiple neurons, that were recorded in response to a stimulus, can provide information about (i.e. encode) the stimulus. Specifically, we are interested in whether the information provided by those neurons about the stimulus is provided redundantly, such that we can obtain it from many (or any) of the recorded neural responses, or whether certain aspects are only present uniquely in individual neurons, but not others; finally, we may find that we need to analyze all neural responses together to decode the stimulus - a case of synergy. All three ways of providing information about the stimulus may coexist and the aim of PID analysis is to determine to what degree each of them is present \cite{wibral2015bits}. 

In this way PID can be used as a framework for systematically testing and comparing theories of neural processing (such as predictive coding \cite{rao1999predictive} or coherent infomax \cite{kay2011coherent}) in terms of their information theoretic "footprint", ~i.e. in terms of the amounts of unique, redundant or synergistic information processing predicted by the theory. The key idea is to identify such theories with a specific information theoretic goal function (e.g. "maximize redundancy while at the same time allowing for a certain degree of unique information"). One may then investigate empirically whether a given neural circuit in fact maximizes the goal function in question or one may use the PID framework to come up with entirely new goal functions \cite{wibral2017partial}. 

The PID problem also arises in cryptography in the context of so called "secret sharing" \cite{Rauh2017}. The idea is that a multiple participants (the sources) each hold some partial information about a particular piece of information called the secret (the target). However, the secret can only be accessed if certain participants combine their information. In this context, PID describes how access to the secret is distributed over the participants.

The partial information decomposition framework has furthermore been used to to operationalize several core concepts in the study of complex and computational systems. These concepts include for instance the notion of information modification \cite{Lizier2013syn_info_mod, wibral2017modification} which has been suggested along with information storage and transfer as one of three fundamental component processes of distributed computation. It has also been proposed that the concepts of emergence and self-organisation can be made quantifiable within the PID framework \cite{Rosas2018self_organisation},\cite{Rosas2020}.

Despite the universality of the PID problem, solutions have only arisen very recently, and the work on consolidating and on distilling them into a coherent structure is still in progress. In this paper we aim to do so by rederiving the theory of partial information decomposition from the perspective of mereology (the study of parthood relations) and formal logic. The general structure of PID arrived at in this way is equivalent to the one originally described by Williams and Beer \cite{williams2010nonnegative}. However, our derivation has the advantage of tackling the problem directly from the perspective of the \textit{parts} into which the information carried by the sources about the target is decomposed, the so called "atoms of information". By contrast, the formulation used until now takes an indirect approach via the concept of redundant information. Furthermore, the approach described here is based on particularly elementary concepts: parthood between information contributions and logical implication between statements about source realizations. 

The remainder of this paper is structured as follows: First, in \S\ref{sec:parthood} we derive the general structure underlying partial information decomposition from considerations of elementary parthood relationships between information contributions. This structure is general in the sense that it still leaves open the possibility for multiple alternative measures of information decomposition. We show that the axioms underlying the formulation by Williams and Beer \cite{williams2010nonnegative, finn2018pointwise} can be proven within the framework described here. In \S\ref{sec:logic_isx} we utilize formal logic to derive a specific PID measure and in this way provide a complete solution to the information decomposition problem.  \S\ref{sec:logic} shows that there is an intriguing connection between formal logic and PID in that the mathematical lattice structure underlying information decomposition is isomorphic to a lattice of logical statements ordered by logical implication. This gives rise to a completely independent exposition of PID theory in terms of a hierarchy of logical constraints on how information about the target can be accessed. In \S\ref{sec:discussion_non_red_PID} we show that the ideas presented here can be utilized to systematically answer the question of whether a (full n-sources) PID can be induced by measures other than redundant information such as synergy or unique information. Before concluding in \S\ref{sec:conclusion}, we briefly address the important distinction between parthood relations and quantitative relations in \S\ref{sec:discussion_parthood_vs_quant}.

\section{The parthood perspective}\label{sec:parthood}
Suppose there are $n$ source variables $S_1,\ldots,S_n$ carrying some joint mutual information ${I(T:S_1,\ldots,S_n)}$ \cite{shannon1948mathematical,cover1999elements} about some target variable $T$ (see Figure \ref{fig:pid_problem_and_xor}, left). The goal of partial information decomposition is to decompose this joint mutual information into its component parts, the so called \textit{atoms} of information. As explained in the introduction, these parts are supposed to represent unique, redundant, and synergistic information contributions. Now, what distinguishes these contributions are their defining part-whole relationships to the information provided by the different source variables: the information uniquely associated with one of the sources in only part of the information provided by \textit{that} source and not part of the information provided by any other source. The information provided redundantly by multiple sources is part of the information carried by \textit{each} of these sources. And the information provided synergistically by multiple source is only part of the information carried by them jointly but not part of the information carried by any of them individually. For this reason, it seems natural to make the part-whole relationship between pieces of information the basic concept of PID. The goal of this section is to make this idea precise, and in this way, to open up a new perspective for thinking about partial information decomposition.

The underlying idea is that any theory should be put on the foundation of as simple and elementary concepts as possible. The part-whole relation is one of the most basic relationships in nature. It appears on all spatial and temporal scales: atoms are parts of molecules, planets are parts of solar systems, the phase of hyperpolarisation is part of an action potential, infancy is part of a human beings life. Moreover, it is not a purely scientific concept but is also ubiquitous in ordinary life: we say for instance, that a prime minister is part of the government or that a slice of pizza is part of the whole pizza. This ubiquity makes it particularly easy to think in terms of part-whole relationships. We hope, therefore, that starting from this vantage point will provide a particularly accessible and intuitive exposition of partial information decomposition. This factor is of particular importance when it comes to the practical application of PID to specific scientific questions and the interpretation of the results of a PID analysis.

Developing the theory of partial information decomposition means that we have to answer three questions:
\begin{enumerate}
\item What are atoms of the decomposition supposed to mean, i.e.~ what \textit{type} of information should they represent?
\item How many atoms are there for a given number of information sources?
\item How large are the different atoms of information given a specific joint probability distribution of sources and target? How many \textit{bits} of information does each atom provide?
\end{enumerate}

In the following sections we will tackle each of these questions in turn.

\begin{figure}[ht] 
\centering
\includegraphics[width=0.7\textwidth]{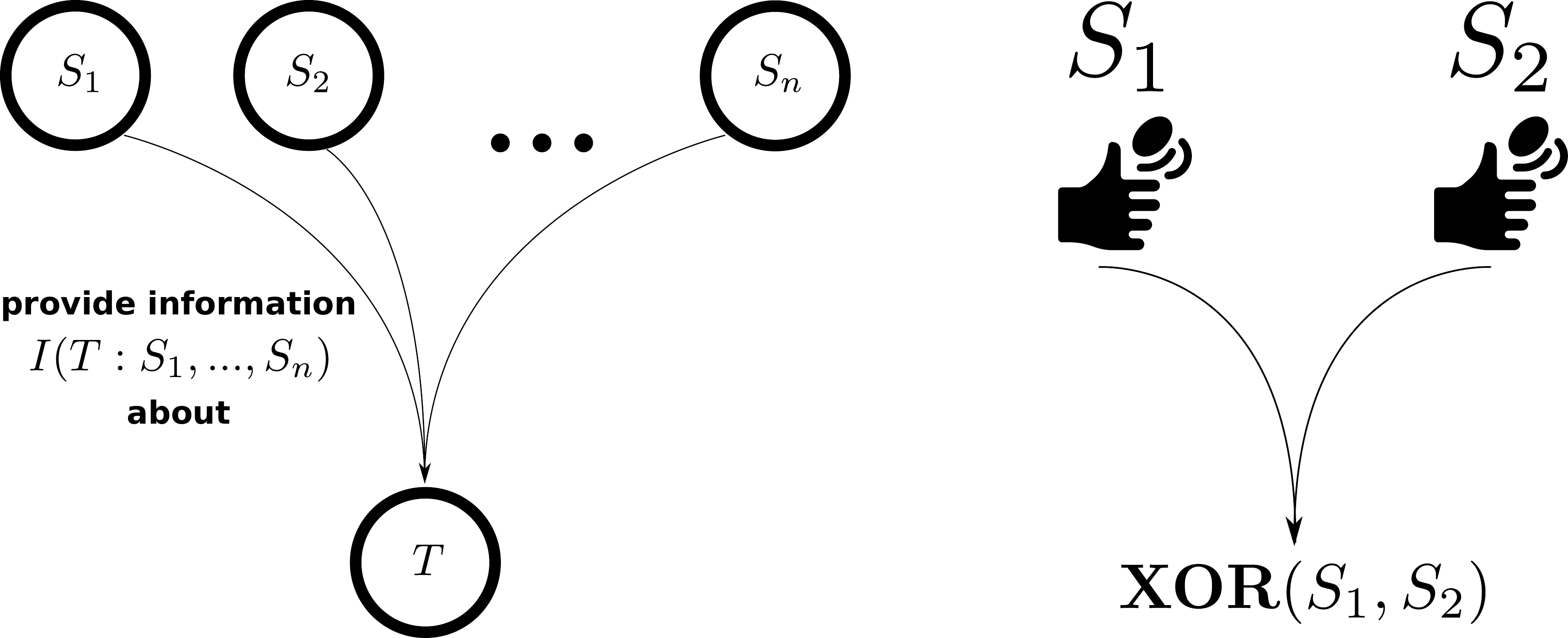}
\caption{Left: The general partial information decomposition problem is to decompose the joint mutual information provided by n source variables $S_1,\ldots,S_n$ about a target variables $T$ into its component parts. Right: Illustration of the exclusive-or example. The sources are two independent coin flips. The target is 0 just in case both coins come up heads or both come up tails. It is 1 if one of the coins is heads while the other is tails. Coin tossing icons made by Freepik, www.flaticon.com.}
\label{fig:pid_problem_and_xor}
\end{figure}

\subsection{What do the atoms of information mean?} \label{sec:meaning_of_atoms}
Asking how to decompose the joint mutual information into its components parts is a bit like asking "How to slice a cake?". Of course, there are many possible ways to do so, and hence, there is no unique answer to the question. In order to make the question more precise we first have to provide a criterion according to which we would like to decompose the joint mutual information. This is what this section is about. What are the atoms of information supposed to mean in the end, i.e. what \textit{type} of information do they represent?

To a first approximation, the core idea underlying the parthood approach to partial information decomposition is to decompose the joint mutual information $I(T:S_1,\ldots,S_n)$ into information atoms, such that each atom is characterized by its parthood relations to the mutual information provided by the different sources. For instance, one atom of information will describe that part of the joint mutual information which is part of the information provided by \textit{each} source, i.e. the information that is redundant to all sources. Another atom will describe the part of the joint mutual information that is only part of the information provided by the first source, i.e. it is unique to the first source. And so on.

Now, we have to refine this idea a bit: it is important to realize that it would not be enough to consider parthood relations to information provided by \textit{individual} sources. The reason is that a \textit{collection} of sources may provide some information that is not contained in any individual source but which only arises by \textit{combining} the information from multiple sources in that collection.  The classical example for this phenomenon is the logical exclusive-or shown in Figure \ref{fig:pid_problem_and_xor}, right. In this example the sources are two independent coin flips. The target is the exclusive-or of the sources, i.e. the target is 0 just in case both coins come up heads or both come up tails, and it is 1 otherwise. Initially, the odds for the target being zero or one respectively are 1:1 because there are four equally likely outcomes in two of which the target is 1 while it is 0 in the other two. Now, if we are told the value of one of the coins, these odds are not affected, and accordingly, we do not obtain any information about the target. For instance, if we are told that the first coin came up heads there are two equally likely outcomes left: Heads-Heads and Heads-Tails. In the first case, the target is zero and in the second case it is one. Hence, the odds are still 1:1. On the other hand, if we are told the value of \textit{both} coins, then we \textit{know} what the value of the target is. In other words, we obtain complete information about the target. 


There are two conclusions to be drawn from examples like this:
\begin{enumerate}
\item There are cases in which multiple information sources combined provide some information that is not contained in any individual source. This type of information is generally called \textit{synergistic information}.
\item Any reasonable theory of information should be compatible with the existence of synergistic information. In particular, it should allow that, in some cases, the information provided jointly by multiple sources is larger than the sum of the individual information contributions provided by the sources. 
\end{enumerate}
Regarding the second point we may note that classical information theory satisfies this constraint because in some cases
\begin{equation}
I(T:S_1,S_2) > I(T:S_1) + I(T:S_2)
\end{equation}
In fact, in the exclusive-or example, each individual source provides zero bits of information while the sources combined provide one bit of information.

Based on these consideration we may rephrase the basic idea of the parthood approach as: we are looking for a decomposition of the joint mutual information into atoms such that each atom is characterized by its parthood relations to the information carried by the different possible \textit{collections} of sources about the target. Of course, we allow collections containing only a single source, such as $\{1\}$, as a special case. Note that we will generally refer to source variables and collections thereof \textit{by their indices}. So instead of writing $\{S_1\}$ and $\{S_1,S_2\}$ to refer to the first source and the collection containing the first and second source, we write $\{1\}$ and $\{1,2\}$ respectively. There are several important technical reasons for this that will become apparent in the following sections. For now it is sufficient to just think of it as a shorthand notation.

Let's now investigate how the idea of characterizing the information atoms by parthood relations plays out in the simple case of two sources $S_1$ and $S_2$. In this case, there are four collections: 
\begin{enumerate}
\item The empty collection of sources $\{\}$
\item The collection containing only the first source $\{1\}$
\item The collection containing only the second source $\{2\}$
\item The collection containing both sources $\{1,2\}$
\end{enumerate}
Now, in order to characterize an information atom $\Pi$ we have to ask for each collection $\bfa$: Is $\Pi$ part of the information provided by $\bfa$? For two of the collections we can answer this question immediately for all $\Pi$: First, no atom of information should be contained in information provided by the empty collection of sources because there is no information in the empty set. If we do not know any source, then we cannot obtain any information from the sources. Second, any atom of information should be contained in the mutual information provided by the full set of sources since this is precisely what we want to decompose into its component parts. Regarding the collections $\{1\}$ and $\{2\}$ we are free to answer yes or no leaving four possibilities as shown in Table \ref{tab:2_sources_parthood_table}.
\begin{table}[ht] 
	\centering
\begin{tabular}{|c| c| c |c |c|} 
	\hline
	Part of & \{\} & \{1\} & \{2\} & \{1,2\} \\ [0.5ex] 
	\hline
	$\Pi_1$ (Synergy) & \textbf{0}  & 0  &0  & \textbf{1}  \\ 
	\hline
	$\Pi_2$ (Unique) & \textbf{0} & 1  & 0& \textbf{1} \\
	\hline
	$\Pi_3$ (Unique) & \textbf{0}  & 0 & 1 & \textbf{1} \\
	\hline
	$\Pi_4$ (Shared) & \textbf{0}& 1  &   1 & \textbf{1} \\
	\hline
\end{tabular}
\caption{Parthood table for the case of two information sources. Each row characterizes a particular atom of information in terms of its parthood relationships with the mutual information provided by the different collections of sources. The bold entries are enforced by the constraints that there is no information in the empty collection of sources and that any piece of information is part of the information carried by the full set of sources about the target.}
\label{tab:2_sources_parthood_table}
\end{table}

The first possibility (first row of Table \ref{tab:2_sources_parthood_table}) is an atom of information that is only part of the information provided by the sources jointly but not part of the information in either of the individual sources. This is the \textit{synergistic information}. The second possibility (second row) is an atom that is part of the information provided by the first source but which is not part of the information in the second source. This atom of information describes the \textit{unique information} of the first source. Similarly, the third possibility (third row) is an atom describing information uniquely contained in the second source. The fourth and last possibility (fourth row) is an atom that is part of the information provided by \textit{each} source. This is the information \textit{redundantly provided} or \textit{shared by}  the two sources.

So based on considerations of parthood we arrived at the conclusion that there should be exactly four atoms of information in the case of two source variables. Each atom is characterized by its parthood relations to the mutual information provided by the different collections of sources. These relationships are described by the rows of Table \ref{tab:2_sources_parthood_table} which we will call \textit{parthood distributions}. Each atom $\Pi$ is formally represented by its parthood distribution $f_{\Pi}$. 

Mathematically, a parthood distribution is a Boolean function from the powerset of $\{1,\ldots,n\}$ to $\{0,1\}$, i.e. it takes a collection of source indices as an input and returns either 0 (the atom described by the distribution is not part of information provided by the collection) or 1 (the atom described by the distribution is part of that information) as an output. But note that not all such functions qualify as a parthood distribution. We already saw that certain constraints have to be satisfied. For instance, the empty set of sources has to be mapped to 0. We propose that there are exactly three constraints a parthood distribution $f$ has to satisfy leading to the following definition
\begin{definition}
A parthood distribution is any function $f:\mathcal{P}\left(\{1,\ldots,n\}\right) \rightarrow \{0,1\}$ such that
\begin{enumerate}
\item $f(\{\}) = 0$ ("There is no information in the empty set")
\item $f(\{1,\ldots,n\}) = 1$ ("All information is in the full set")
\item For any two collections of source indices $\mathbf{a}$, $\mathbf{b}$: If $\mathbf{b} \supseteq \mathbf{a}$, then $ f(\mathbf{a}) = 1 \Rightarrow f(\mathbf{b}) = 1$ (Monotonicity)
\end{enumerate}
\end{definition} 
The third constraint says that if an atom of information is part of the information provided by some collection of sources $\mathbf{a}$, then it also has to be part of the information provided by any superset of this collection. For example, if an atom is part of the information in source 1, then it also has to be part of the information in sources 1 and 2 combined. Note that this monotonicity constraint only matters if there are more than two information sources. Otherwise it is implied by the first two constraints. To fix ideas, an example of a Boolean function that is \textit{not} a parthood distribution is shown in Table \ref{tab:example_of_non_parthood_distribution}. The function assigns a 1 to the collection $\{1\}$ but a 0 to collections $\{1,2\}$ and $\{1,3\}$ which are supercollections of $\{1\}$. Thus, there can be no atom of information with the parthood relations described by this Boolean function.

\begin{table}[ht]
\centering
\begin{tabular}{|c |c |c |c |c |c |c | c|c |} 
	\hline
	Part of & \{\} & \{1\} & \{2\} & \{3\} &  \{1,2\} & \{1,3\} &\{2,3\} & \{1,2,3\}\\ [0.5ex] 
	\hline
	& 0   & 1 & 0 & 0  & $\mathbf{0}$ & $\mathbf{0}$  & 0  & 1 \\ 
	\hline
\end{tabular}
\caption{Example of Boolean function that is not a parthood distribution. Bold entries violate the monotonicity constraint.}
\label{tab:example_of_non_parthood_distribution}
\end{table}

We may now answer the question about the meaning of the atoms of information, i.e. what \textit{type} of information they represent: They represent information that is part of the information provided by certain collections of sources but not part of the information of other collections. More precisely we can phrase this idea in terms of the following core principle:

\begin{principle}\label{prc:core_principle_1}
Each atom of information is characterized by a parthood distribution describing whether or not it is part of the information provided by the different possible collections of sources. The atom $\Pi(f)$ with parthood distribution f is exactly that part of the joint mutual information about the target which is 1) part of the information provided by all collections of sources $\mathbf{a}$ for which $f(\mathbf{a}) = 1$, and 2), which is not part of the information provided by collections for which $f(\mathbf{a}) = 0$.
\end{principle}

Given this characterization of the information atoms we are now in a position to answer the second question: How many atoms are there for a given number of information sources.

\subsection{How many atoms of information are there?} \label{subsec:parthood_number_of_atoms}
Since each atom is characterized by its parthood distribution, the answer is straightforward: there is one atom per parthood distribution, or in other words, one atom per Boolean function satisfying the constraints presented in the previous section. The monotonicity constraint turns out to be most restrictive. In fact, once the monotonicity constraint is satisfied the other two constraints only rule out one Boolean function each as shown in Table \ref{tab:constant_monotone_functions}. The reason is the following: Firstly, there is only a single \textit{monotonic} Boolean function that assigns the value 1 to the empty set, namely, the function that is always 1. Since the empty set is subset of any other set, monotonicity enforces to assign a 1 to all sets once the empty set has value 1. However, this possibility is ruled out by the first constraint saying that there is no information in the empty set. Secondly, there is only a single \textit{monotonic} Boolean function assigning the value 0 to the full set $\{1,\ldots,n\}$, namely the function that is always 0. Since any other set of source indices is contained in the full set, monotonicity forces us to assign a 0 to all sets once the full set has value 0. If we were to assign a 1 to any other set, then we would have to assign a 1 to the full set as well.

\begin{table}[ht]
\centering
\begin{tabular}{|c |c |c |c |c |c |} 
	\hline
	Part of & \{\} & \ldots & \ldots & \ldots  & \{1,\ldots,n\}\\ [0.5ex] 
	\hline
	& 1  &  $\mathbf{1}$ & $\mathbf{1}$ & $\mathbf{1}$  & $\mathbf{1}$  \\ 
	&   $\mathbf{0}$ &  $\mathbf{0}$  & $\mathbf{0}$  & $\mathbf{0}$   & 0  \\ 
	\hline
\end{tabular}
\caption{The two constant Boolean functions are ruled out by the first and second constraint on parthood distributions described above.}
\label{tab:constant_monotone_functions}
\end{table}

This means that the number of atoms is equal to \textit{the number of monotonic Boolean functions minus two}. Now the sequence of the numbers of monotonic Boolean functions of n-bits is a very famous sequence in combinatorics called the \textit{Dedekind numbers}. The Dedekind numbers are a very rapidly (in fact super-exponentially) growing sequence of numbers of which only the first eight entries are known to date \cite{stanley1997enumerative}. The values for $2\leq n \leq 6$ of the Dedekind numbers are: 6, 20, 168, 7581, 7828354.

Now that we have answered what type of information the different atoms represent and how many there are for a given number of information sources, there is one important question left: How large are these different atoms? How many \textit{bits} of information does each atom provide?


\subsection{How large are the atoms of information?} \label{sec:parthood_how_large}
The question of the sizes of the atoms is not a trivial one since the number of atoms grows so quickly. In the case of four information sources there are already 166 atoms. Hence, it does not appear to be feasible to define the amount of information of each of these atoms separately. What we need is a systematic approach that somehow fixes the sizes of all atoms at the same time. The core idea is to transform the problem into a much simpler one in which only a single type of informational quantity has to be defined. In the following we show how this can be achieved in three steps.

\subsubsection{Define a quantitative relationship between atoms and composite quantities}\label{sec:parthood_step_1}
So far we have only discussed how the atoms of information relate \textit{qualitatively} to composite information quantities that are made up of multiple atoms, in particular mutual information (in the next section we will encounter another non-atomic quantity). We saw for instance, that in the case of two sources, the mutual information contributions provided by the individual sources, $I(T:S_1)$ and $I(T:S_2)$, each consist of a unique and a redundant information atom, while the joint mutual information $I(T:S_1,S_2)$ additionally consists of a synergistic part. This is illustrated in the information diagram shown in Figure \ref{fig:two_sources_PID_illustration}.
 \begin{figure}[ht] 
	\centering
	\includegraphics[width = 0.4\textwidth]{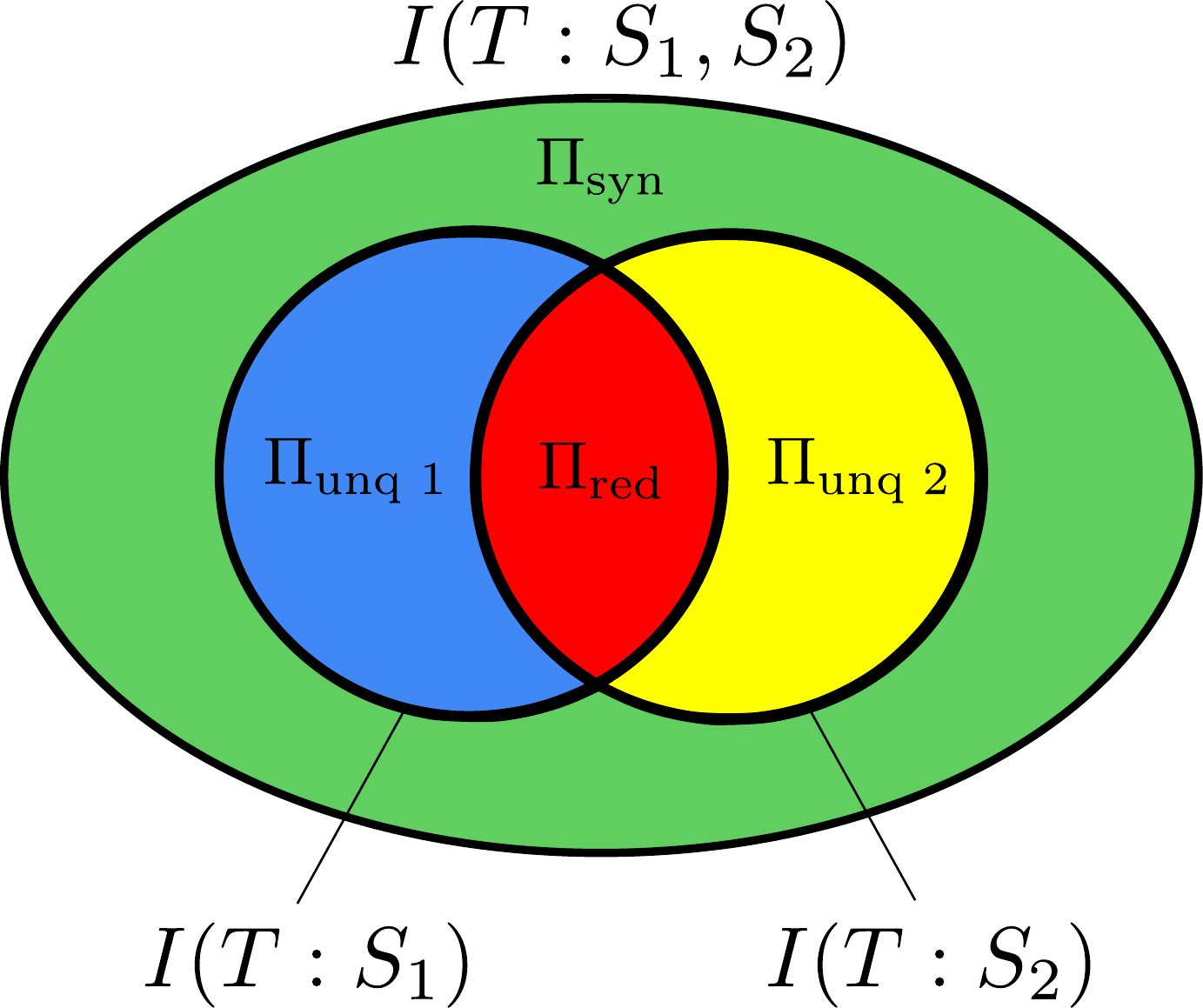}
	\caption{Information diagram depicting the partial information decomposition for the case of two information sources. The inner two black circles represent the mutual information provided by the first source (left) and the second source (right) about the target. Each of these mutual information terms contains two atomic parts: $I(T:S_1)$ consists of the unique information in source 1 ($\Pi_{\text{unq 1}}$, blue patch) and the information shared with source 2 ($\Pi_{\text{red}}$, red patch).  $I(T:S_2)$ consists of the unique information in source 2 ($\Pi_{\text{unq 2}}$, yellow patch) and again the shared information. The joint mutual information $I(T:S_1,S_2)$ is depicted by the large black oval encompassing the inner two circles. $I(T:S_1,S_2)$ consists of four atoms: The unique information in source 1 ($\Pi_{\text{unq 1}}$, blue patch), the unique information in source 2 ($\Pi_{\text{unq 2}}$, yellow patch), the shared information ($\Pi_{\text{red}}$, red patch), and additionally the synergistic information ($\Pi_{\text{syn}}$, green patch).}
	\label{fig:two_sources_PID_illustration}
\end{figure}

Now the question arises: How are these mutual information terms related to the atoms they consist of \textit{quantitatively}? The most straightforward answer (and the one generally accepted in the PID field) is that the mutual information is simply the \textit{sum} of the atoms it consists of. We propose to extend this principle to any composite information quantity, i.e. any quantity that can be described as being made up out of multiple information atoms:
\begin{principle}\label{prc:wholes_are_sums_of_parts}
The size of any non-atomic information quantity (i.e. the amount of information it contains) is the sum of the sizes of the information atoms it consists of.
\end{principle}
\noindent
We could also rephrase this as "wholes are the sums of their (atomic) parts". In the case of two information sources, this principle leads to the following three equations:
{\small
\begin{align}
&I(T:S_1, S_2) \! = \! \Pi_{\text{red}} \! + \! \Pi_{\text{unq 1}}\! + \! \Pi_{\text{unq 2}} \! + \! \Pi_{\text{syn}} \\ 
&I(T:S_1) \! =  \! \Pi_{\text{red}} \! + \! \Pi_{\text{unq 1}} \\ 
&I(T:S_2) \! = \!  \Pi_{\text{red}}\! + \! \Pi_{\text{unq 2}}
\end{align}
}

This already gets us quite far in terms of determining the sizes of the atoms: The sizes of the atoms are the solutions to a linear system of equations. The only problem is that the system is underdetermined. We have four unknowns but only three equations. In the case of three sources, the problem is even more severe. In this case, there are seven non-empty collections of sources, and hence, seven mutual information terms. Again each of these terms is the sum of certain atoms. But as shown in section \S\ref{sec:parthood}b there are 18 atoms. So we are short of 11 equations! 

In general the equations relating the mutual information provided by some collection of sources $\mathbf{a}$ and the information atoms can be expressed easily in terms of their parthood distributions:
\begin{equation}\label{eq:quant_relation_atoms_MI}
I(T:\mathbf{a}) = \sum\limits_{f(\mathbf{a})= 1} \Pi(f)
\end{equation}
where $\Pi(f)$ is the information atom corresponding to parthood distribution $f$ and the summation notation means that we are summing over all f such that $f(\mathbf{a})=1$. Note that on the left-hand-side we are using the shorthand notation $I(T:\mathbf{a})$ for the mutual information $I(T:(S_i)_{i\in \mathbf{a}})$ provided by the collection $\mathbf{a}$. Equation \eqref{eq:quant_relation_atoms_MI} can be taken to define a minimal notion of a partial information decomposition, i.e.~any set of quantities $\Pi(f)$ at least has to satisfy this equation in order to be considered a partial information decomposition (or at least to be considered a parthood-based / Williams and Beer type PID). For a formal definition of such a minimally consistent PID see Appendix \ref{app:min_cons_pid}.

This concludes the first step. The next one is to find a way to come up with the appropriate number of additional equations. In doing so we will follow the same approach as Williams and Beer and utilize the concept of \textit{redundant information} to introduce additional constraints. It should be noted that this is not the only way to derive a solution for the information atoms. In other words, a PID does not have to be "redundancy based". This issue is discussed in detail in \S\ref{sec:discussion_non_red_PID}. For now, however, let us follow the conventional path and see how it enables us to determine the sizes of the atoms of information.


\subsubsection{Formulate additional equations using the concept of redundant information} \label{sec:parthood_step_2}
The basic idea is now to extent the considerations of the previous step to another composite information quantity: the redundant information provided by multiple collections of sources about the target which we will generically denote by $I_\cap(T:\mathbf{a}_1,\ldots,\mathbf{a}_m)$. The $\cap$-symbol refers to the idea that the redundant information of collections $\mathbf{a}_1,\ldots,\mathbf{a}_m$ is the information contained in $\mathbf{a}_1$ \textit{and} $\mathbf{a}_2$ \textit{and}, \ldots, \textit{and} $\mathbf{a}_m$. Intuitively, given two collections of sources $\mathbf{a}_1$ and $\mathbf{a}_2$, their redundant information is the information ``shared'' by those collections, what they have "in common", or geometrically: their overlap. These informal ideas are illustrated on the left side in Figure \ref{fig:redundancy_illustration}. 

 \begin{figure}[ht] 
	\centering
	\includegraphics[width=0.75\textwidth]{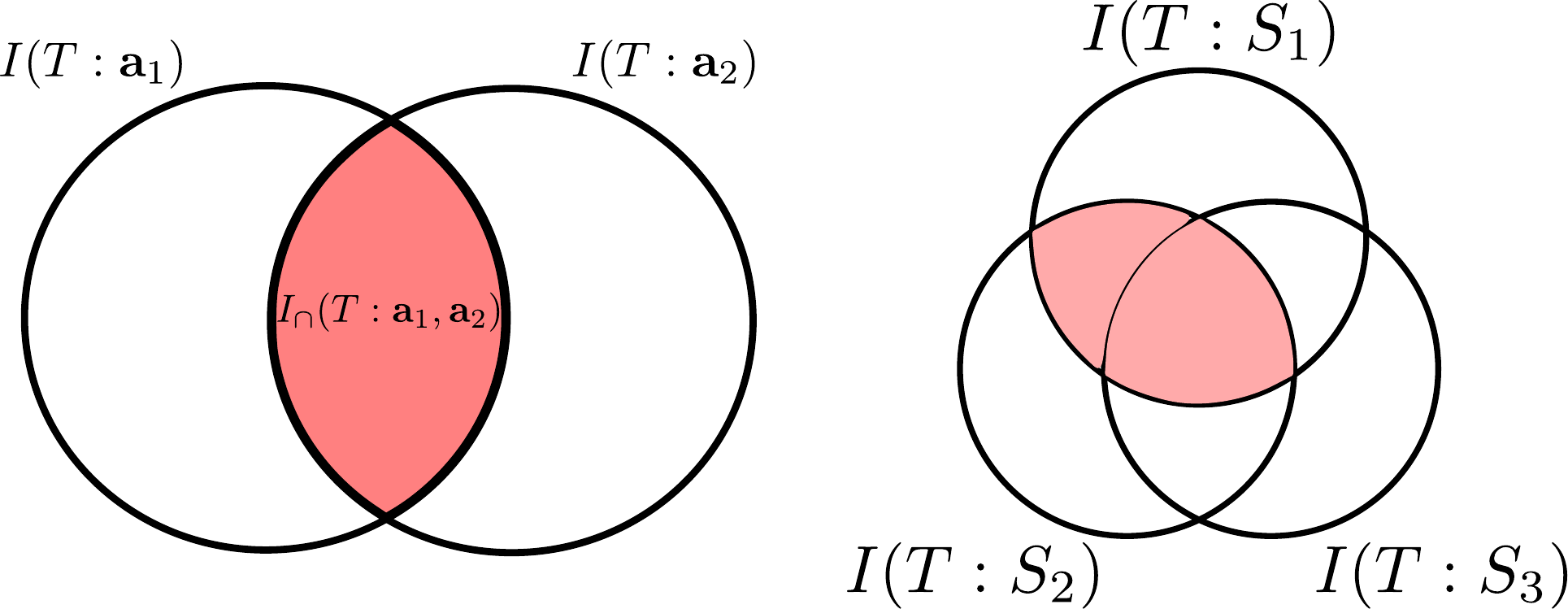}
	\caption{Left: Illustration of the idea of the redundant information of collections $\mathbf{a}_1$ and $\mathbf{a}_2.$ Right: Redundant information is generally not an atomic quantity. In the context of three information sources, the redundant information of sources 1 and 2 consists of two parts: the information shared by \textit{only} by sources 1 and 2, and the information shared by all three sources. }
	\label{fig:redundancy_illustration}
\end{figure}

Note that the redundant information of multiple collections of information sources is not defined in classical information theory. We have to come up with an appropriate measure of redundant information ourselves. However, the informal ideas just describes already tell us that redundant information, no matter how we define it, should be related qualitatively to the information atoms in a very specific way: the information redundantly provided by multiple collections of sources should consist of exactly those information atoms that are part of the information carried by \textit{all} of those collections:

\begin{principle}\label{prc:parthood_criterion_red}
The redundant information ${I_\cap(T:\mathbf{a}_1,\ldots,\mathbf{a}_m)}$ consists of all information atoms that are part of the information provided by \textit{each} $\mathbf{a}_i$, i.e. all atoms with a parthood distribution satisfying $f(\mathbf{a}_i)  = 1$ for all $i = 1,\ldots,m$. 
\end{principle}
Let's see what this principle implies in concrete examples. We saw that in the case of two sources, the redundant information of source 1 and source 2, $I(T:\{1\},\{2\})$, is actually itself an atom, namely the atom with the parthood distribution

\begin{table}[h] 
	\centering
	\begin{tabular}{| c| c |c |c|} 
		\hline
		\{\} & \{1\} & \{2\} & \{1,2\} \\ [0.5ex] 
		\hline
	 	0  & 1  & 1  & 1 \\  
		\hline
	\end{tabular}
\end{table}

This is the only atom that is part of both the information provided by the first source and also part of the information provided by the second source. But this is really a special case. Note what happens if we add a third source to the scenario. In this case the redundant information $I(T:\{1\},\{2\})$ of sources 1 and 2 should consist of \textit{two} parts: First, the information shared by \textit{all three} sources (which is certainly also shared by sources 1 and 2), and secondly, the information shared \textit{only} by sources 1 and 2 but not by source 3. This is illustrated on the right side in Figure \ref{fig:redundancy_illustration}. Note also that in the case of three sources there are actually \textit{many} redundancies that we may compute: 

\begin{enumerate}
\item the redundancy of all three sources ${I_\cap(T:\{1\},\{2\},\{3\})}$.
\item  the redundancy of any \textit{pair} of sources such as the redundancy of  $I_\cap(T:\{1\},\{2\})$.
\item the redundancy between a single source and a pair of sources such as  $I_\cap(T:\{1\},\{2,3\})$.
\item the redundancy between two pairs of sources such as $I_\cap(T:\{1,2\},\{2,3\})$.
\item the redundancy of all three possible pairs of sources  $I_\cap(T:\{1,2\},\{1,3\},\{2,3\})$.
\end{enumerate}
It turns out that in total there are 11 redundancies (strictly speaking we should say 11 "proper" redundancies as will be explained below). But this is exactly the number of missing equations in the case of three information sources (see last paragraph of previous section).

Now, combining Core Principles \ref{prc:wholes_are_sums_of_parts} and \ref{prc:parthood_criterion_red}, allows us the answer what the \textit{quantitative} relationship between redundant information and information atoms has to be: the redundant information of collections of sources $\mathbf{a}_1,\ldots,\mathbf{a}_m$ is the sum of all atoms that are part of the information provided by \textit{each} collection:
\begin{equation}\label{eq:quantitative_relation_red_atoms}
I_\cap(T:\mathbf{a}_1,\ldots,\mathbf{a}_m) = \sum\limits_{f(\mathbf{a}_i)  = 1 \forall i = 1,\ldots,m} \Pi(f)
\end{equation}
where again the notation means that we are summing over all f that satisfy the condition below the summation sign. This equation can be read in two ways: First, as placing a constraint on the redundant information $I_\cap$, namely that it has to be the sum of specific atoms. This means that if we already knew the sizes of the $\Pi$'s, we could compute $I_\cap$. However, the sizes of the $\Pi$'s are precisely what we are trying to work out. Now the crucial idea is that we can also read the equation the other way around: if we can come up with some reasonable measure of redundant information $I_\cap$ we may be able to \textit{invert} equation \ref{eq:quantitative_relation_red_atoms} in order to obtain the $\Pi$'s. So the final step will be to show that such an inversion is in fact possible and will lead to a unique solution for the atoms of information. 

Before proceeding to this step, it is important to briefly clarify the relationships between the three central concepts we have discussed so far: 

\begin{enumerate}
\item the mutual information (the quantity we want to decompose)
\item the information atoms (the quantities we are looking for)
\item redundant information (the quantity we are going to use to find the information atoms)
\end{enumerate}

These concept are easily confused with each other but should be clearly separated. The relationships between them are shown in Figure \ref{fig:relation_atoms_MI_Red}. First, based on what we have said so far, mutual information can be shown to be a special case of redundant information: the redundant information of a single collection $I_\cap(T:\mathbf{a}_1)$, i.e. "the information the collection shares \textit{with itself} about the target". The reason for this is that Core Principle \ref{prc:parthood_criterion_red} tells us that the redundant information of a single collection consists of all the atoms that are part of the mutual information carried by \textit{that} collection about the target. But this is simply the mutual information of that collection:
\begin{align}
I_\cap(T:\mathbf{a}_1) &\stackrel{\text{Eq. }\ref{eq:quantitative_relation_red_atoms}}{=} \sum\limits_{f(\mathbf{a}_i)  = 1 \forall i = 1,\ldots,m} \Pi(f) 
= \sum\limits_{f(\mathbf{a}_1) = 1}  \Pi(f) 
\stackrel{\text{Eq. }\ref{eq:quant_relation_atoms_MI}}{=} I(T:\mathbf{a}_1) 
\end{align}
Accordingly, mutual information has been called "self-redundancy" in the PID literature  (although not based on parthood arguments) \cite{williams2010nonnegative}. The relationship between redundant information and atoms is as follows: Only the "all-way" redundancy, i.e. the information shared by \textit{all} $n$ sources is itself an atom. Any other redundancy, such as the redundancy of only a subset of sources, is a composite quantity made up out of multiple atoms.

\begin{figure}[ht] 
\centering
\includegraphics[width=0.4\textwidth]{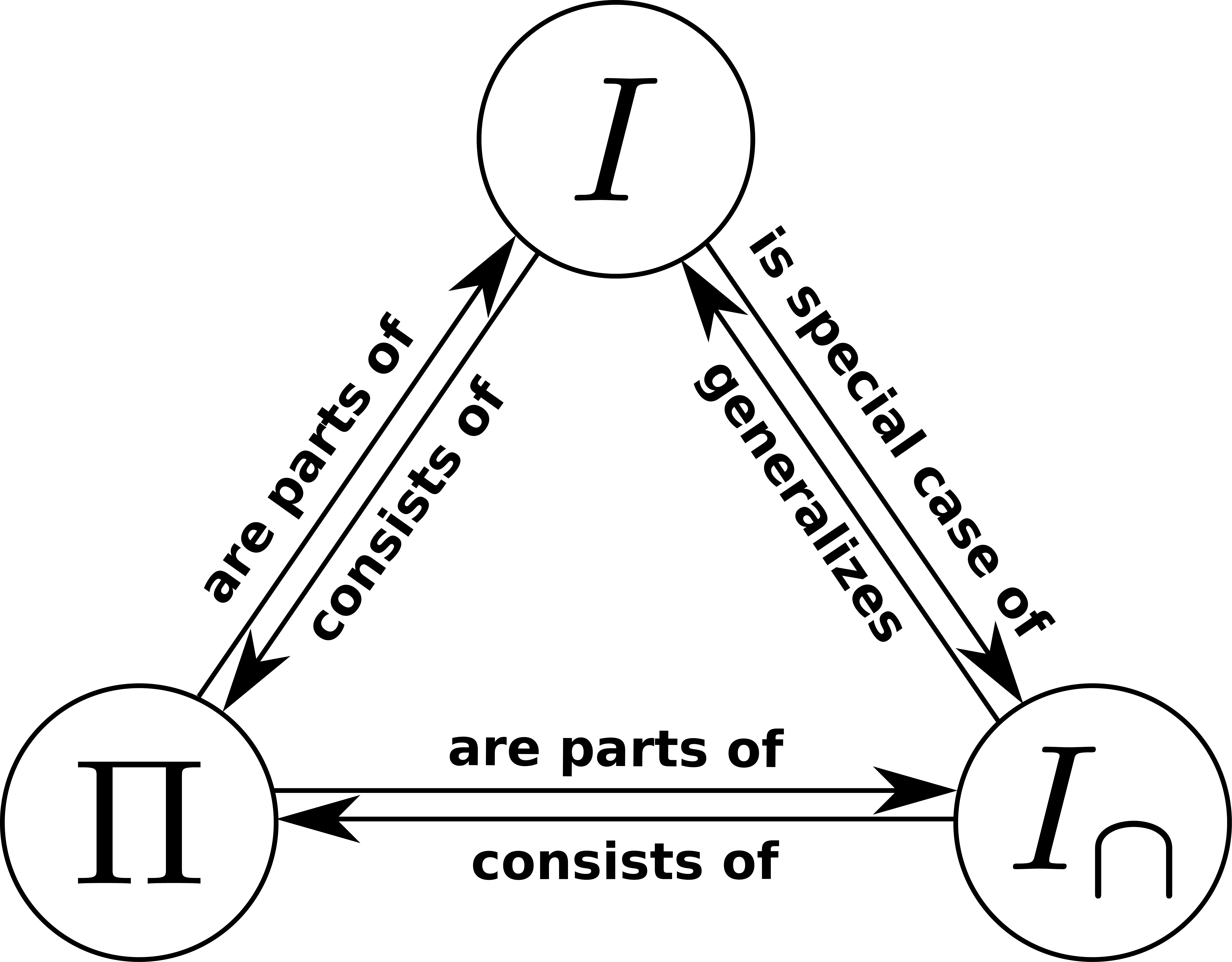}
\caption{Relationships between mutual information, redundant information, and information atoms.}
\label{fig:relation_atoms_MI_Red}
\end{figure}


\subsubsection{Show that a measure of redundant information leads to a unique solution for the information atoms} \label{sec:parthood_step_3}
There is a very useful fact about parthood distributions that will help us to obtain a unique solution for the atoms given an appropriate measure of redundant information: parthood distributions can be ordered in a very natural way into a lattice structure that is tightly linked to the idea of redundancy. The lattice for the case of three sources is shown in Figure \ref{fig:parthood_3_lattice}. The parthood distributions are ordered as follows: If there is a 1 in certain positions on a parthood distribution $f$, then all the parthood distributions g below it also have a 1 in the same positions, plus some additional ones. Or in terms of the atoms corresponding to these parthood distributions: If an atom $\Pi(f)$ is part of the information provided by some collections of sources, then all the atoms $\Pi(g)$ below it are also part of the information provided by these collections.  Formally, we will denote this ordering by $\sqsubseteq$ and it is defined as
\begin{equation}
f \sqsubseteq g \Leftrightarrow (f(\mathbf{a}) = 1 \rightarrow g(\mathbf{a}) = 1 \!\text{ for any }\! \mathbf{a} \subseteq \{1,\ldots,n\})
\end{equation}
For $n$ information sources we will denote the lattice of parthood distributions by $(\mathcal{B}_n, \sqsubseteq)$, where $\mathcal{B}_n$ is the set of all parthood distribution in the context of $n$ sources (for proof that this structure is in fact a lattice in the formal sense see Appendix \ref{app:proof_isomorph}. 

Note that the different "levels" of the lattice contain parthood distributions with the same number of ones and that higher level parthood distributions contain \textit{less} ones: At the very top in Figure \ref{fig:parthood_3_lattice}, there is the parthood distribution describing the atom that is \textit{only} part of the joint mutual information provided by all three sources combined, i.e. the synergy of the three sources. One level down, there are the three parthood distributions that assign the value 1 exactly two times. Yet another level down, we find the three possible parthood distributions that assign the value 1 exactly three times. And so on and so forth until we reach the bottom of the lattice which corresponds to the information shared by all three sources. Accordingly the corresponding parthood distribution assigns the value 1 to all collections (except of course the empty collection).

\begin{figure}[ht] 
	\centering
	\includegraphics[width=0.33\textwidth]{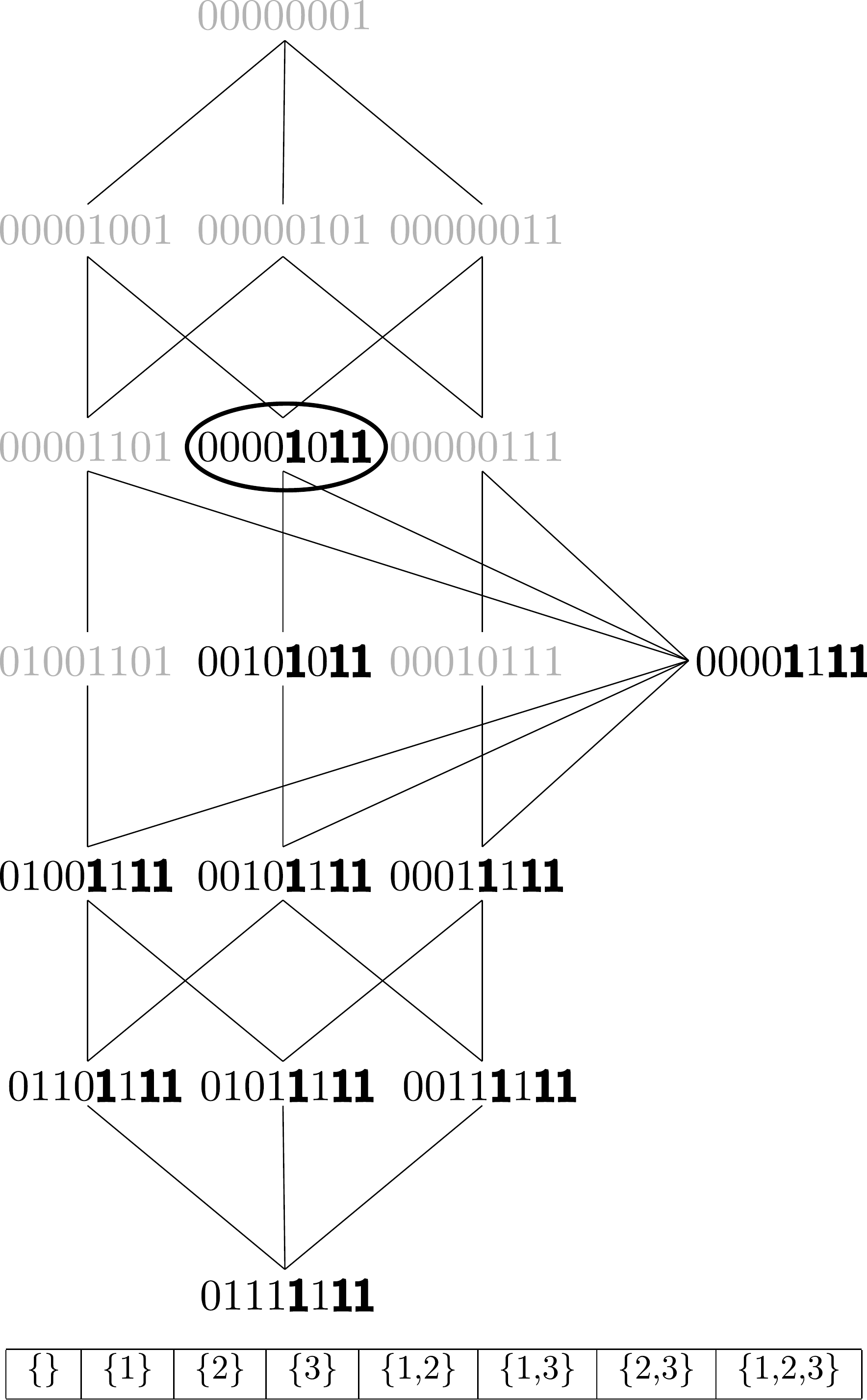}
	\caption{Lattice of parthood distributions for the case of three information sources. The parthood distributions are represented as bit-strings where the $i$-th bit is the value that the parthood distribution assigns to the $i$-th collection of sources. The order of these collections is shown below the lattice for reference. A distribution $f$ is below a distribution $g$ just in case $f$ has value 1 in the same positions as g and in some additional positions. This is illustrated for the parthood distribution highlighted by the black circle. The positions in which it assigns the value 1 are marked in bold face.}
	\label{fig:parthood_3_lattice}
\end{figure}

Ordering all the parthood distributions (and hence atoms) into such a lattice provides a good overview that tells us how many atoms exist for a given number of source variables and what their characteristic parthood relationships are. But the lattice plays a much more profound role because it is very closely connected to the concept of redundant information. The idea is to associate with each parthood distribution in the lattice a particular redundancy: the redundant information of all the collections that are assigned the value 1 by the distribution. In other words, for any parthood distribution $f$ we consider the redundancy
\begin{equation}\label{eq:I_cap_of_f}
I_\cap(T:f) := I_\cap\left(T:(\mathbf{a} \mid f(\mathbf{a})=1)\right)
\end{equation}
For example, in the case of three sources, the redundant information associated with the parthood distribution that assigns value 1 to collections $\{1,2\}$, $\{2,3\}$, and $\{1,2,3\}$, and value 0 to all other collections (the one emphasized in Figure \ref{fig:parthood_3_lattice}), is simply ${I_\cap(T:\{1,2\},\{2,3\},\{1,2,3\})}$. We saw in the previous section that any redundancy $I_\cap(T:\bfa_1,\ldots,\bfa_m)$ is the sum of all atoms that are part of the information provided by each of the $\bfa_i$. Now here is the connection between the lattice and redundant information: these atoms are the ones that have value 1 on each $\mathbf{a}_i$. But, by definition of the ordering, these are precisely the ones corresponding to parthood distributions \textit{below and including} the parthood distribution for which we are computing the associated redundancy. In other words, the redundant information associated with a parthood distribution $f$ can always be expressed as
\begin{equation}\label{eq:moebius_inversion}
I_\cap(T:f) = \sum\limits_{g \sqsubseteq f}\Pi(g)
\end{equation} 
In this way we obtain one equation per parthood distribution. And since there are as many information atoms as parthood distributions, we obtain as many equations as unknowns. This is already a good sign. But is a unique solution for the information atoms guaranteed? This question can be answered affirmatively by noting that the system of equations described by \eqref{eq:moebius_inversion} (one equation per $f$) is not just any linear system, but has a very special structure: one function $I_\cap(T:f)$ evaluated at a point $f$ on a lattice is the sum of another function $\Pi(f)$ over all points on the lattice below and including the point f. The process of solving such a system for the $\Pi(f)$'s once all the $I_\cap(T:f)$'s are given, or in other words \textit{inverting} equation \eqref{eq:moebius_inversion}, is called \textit{Moebius Inversion}. Crucially, a unique solution is guaranteed for any real or even complex valued function $I_\cap$ that we may put on the lattice \cite{tittmann2014einfuhrung}.

This means that we have now completely shifted the problem of determining the sizes of the information atoms to the problem of coming up with a reasonable definition of redundant information $I_\cap(T:f)$. Even though we have to define this quantity for \textit{each} parthood distribution $f$ this is still a much simpler task. The reason is that all the $I_\cap$'s represent exactly the same \textit{type} of information, namely redundant information. On the other hand, the information atoms $\Pi$ represent completely different types of information. Even in the simplest case of two sources we have to deal not only with redundant information, but also unique information and synergistic information.  And the story gets more and more complicated the more information sources are considered. 

Now, note that apparently we only need to define quite special redundant information terms, namely the redundancies associated with parthood distributions $I_\cap(T:f)$ (see definition \eqref{eq:I_cap_of_f}). However, we will now show that these are in fact \textit{all} possible redundancies, i.e. the redundancy of any tuple of collections of sources $\mathbf{a}_1,\ldots,\mathbf{a}_m$ is necessarily equal to a redundancy associated with a specific parthood distribution. The reason for this is that the quantitative relation between atoms and redundant information (equation \eqref{eq:quantitative_relation_red_atoms}) not only provides a way to solve for the information atoms once we know $I_\cap$, it also implies that $I_\cap$ has to satisfy the following invariance properties:
\begin{enumerate}
	\item $I_\cap(T:\mathbf{a}_1,\ldots,\mathbf{a}_m) = I_\cap(T:\mathbf{a}_{\sigma(1)},\ldots,\mathbf{a}_{\sigma(m)})$ for any permutation $\sigma$ (\textbf{symmetry})
	\item If $\mathbf{a}_i = \mathbf{a}_j$ for $i\neq j$, then ${I_\cap(T:\mathbf{a}_1,\ldots,\mathbf{a}_m)}$ $=$ ${I_\cap(T:\mathbf{a}_1,\ldots,\mathbf{a}_{i-1},\mathbf{a}_{i+1},\ldots, \mathbf{a}_m)}$  \\ \textbf{ (idempotency)}
	\item If $\mathbf{a}_i \supset \mathbf{a}_j$ for $i\neq j$, then ${I_\cap(T:\mathbf{a}_1,\ldots,\mathbf{a}_m)}$ $=$ ${I_\cap(T:\mathbf{a}_1,\ldots,\mathbf{a}_{i-1},\mathbf{a}_{i+1},\ldots, \mathbf{a}_m)}$ \textbf{ (invariance under superset removal / addition)}
	\item $I_\cap(T:\mathbf{a}) = I(T:\mathbf{a})$ (\textbf{self-redundancy})
\end{enumerate}

We can easily ascertain that any measure of redundant information $I_\cap$ has to have these properties by taking a closer look at the condition describing which atoms to sum over in order to obtain a particular redundant information term $I(T:\mathbf{a}_1,\ldots, \mathbf{a}_m)$: we have to sum over the atoms with parthood distribution satisfying $f(\mathbf{a}_i) = 1$ for all $i=1,\ldots,m$. Now whether or not this condition is true of a given parthood distribution $f$, first, does not depend on the \textit{order} in which the collections $\mathbf{a}_i$ are given (symmetry), secondly, it does not depend on whether the same collection $\mathbf{a}$ is repeated multiple times (idempotency), and thirdly, it does not matter whether we add or remove some collection $\mathbf{a}_i$ that is a proper superset of some other collection (superset removal/addition). This fact is due to the monotonicity constraint on parthood distributions. Finally, the "self-redundancy" property was already established in the previous section.

These invariance properties are referred in the literature as the Williams and Beer axioms for redundant information \cite{finn2018pointwise} (in addition there is a \textit{quantitative} monotonicity axiom that we reject. See \S\ref{sec:discussion_parthood_vs_quant}). However, in the parthood formalism described here they are not themselves axioms but are \textit{implied} by the core principles we have set out. The first two invariance properties imply that we may restrict ourselves to \textit{sets} instead of tuples of collections in defining $I_\cap$. The third constraint additionally tells us that we can restrict ourselves to those sets of collections $\{\mathbf{a}_1,\ldots,\mathbf{a}_m \}$ such that no collection $\mathbf{a}_i$ is a superset of another collection $\mathbf{a}_j$. Such sets of collections are called \textit{antichains}. Hence, the redundancy of \textit{any} tuple of collections of sources $\mathbf{a}_1,\ldots,\mathbf{a}_m$ is necessarily equal to the redundancy associated with a particular antichain. This antichain results from ignoring the order and repetitions of the $\mathbf{a}_i$, and removing any supersets. For instance, $I_\cap(T:\{1\},\{1\}, \{2\}, \{1,2\}) = I_\cap(T:\{1\},\{2\})$. 

We can now see that the redundancies $I_\cap(T:f)$ are in fact all possible redundancies by associating with any antichain $\alpha = \{\mathbf{a}_1,\ldots,\mathbf{a}_m\}$ a parthood distribution $f_\alpha$ that assigns the value 1 to all $\mathbf{a}_i$ \textit{and all supersets of these collections}, while it assigns the value 0 to all other collections. Now, due to the invariance of $I_\cap$ under removal of supersets, it immediately follows that $I_\cap(T:f_\alpha) = I_\cap(T:\alpha)$. So in conclusion, there is one redundancy for each antichain $\alpha$ and these redundancies are equal to the redundancies associated with the corresponding parthood distributions. Hence the redundancies $I_\cap(T:f)$ are in fact \textit{all} possible redundancies. 

Of course, there is also an inverse mapping associating with any parthood distribution f an antichain $\alpha_f$. In fact, the lattice of parthood distributions $(\mathcal{B}_n,\sqsubseteq)$ is \textit{isomorphic} to a lattice of antichains $(\mathcal{A}_n, \preceq)$ equipped with an ordering relationship that was originally introduced by Crampton and Loizou \cite{Crampton2000} and used by Williams and Beer in their original exposition of PID. The formal proof of this fact is postponed to section \S\ref{sec:logic} where a third perspective on PID, the logical perspective, is introduced.

In the next section, we will tackle the problem of defining a measure of redundant information for each parthood distibution / antichain by connecting PID theory to formal logic. The measure $I_\cap^\text{sx}$ derived in this way is identical to the one described in \cite{makkeh2020isx}. In showing how this measure can be inferred from logical- and parthood-principles we aim to 1) strengthen the argument for $I_\cap^\text{sx}$, and 2), open the gateway between PID-theory and formal logical. This latter point is elaborated in  \S\ref{sec:logic}.


\section{Using logic to derive a measure of redundant information}\label{sec:logic_isx}
We have now solved the PID problem up to specifying a reasonable measure of redundant information $I_\cap$ between collections that form an antichain. In this section, we will derive such a measure. In doing so we will first move from the level of random variables $T,S_1,\ldots,S_n$ to the level of particular realizations $t,s_1,\ldots,s_n$ of these variables. This level of description is generally called the \textit{pointwise} level and has been used as the basis of classical information theory by Fano \cite{fano1961transmission}. Pointwise approaches to PID have been put forth by \cite{finn2018pointwise} and \cite{makkeh2020isx}.

Note that moving to the level of realizations simplifies the problem considerably because realizations are much simpler objects than random variables. A realization is simply a specific
symbol or number whereas a random variables is an object that  may  take  on  various  different  values  and  can only  be  fully  described  by  an  entire  probability  distribution over these values.

\subsection{Going Pointwise}\label{sec:logic_pointwise}
The idea underlying the pointwise approach is to consider the information provided by a particular joint realization (observation) of the source random variables about a particular realization (observation) of the target random variable. So from now on we assume that these variables have taken on \textit{specific} values $s_1,\ldots,s_n,t$. It was shown by Fano \cite{fano1961transmission} that the whole of classical information theory can be derived from this pointwise level. By placing a certain number of reasonable constraints or axioms on pointwise information, it follows that this information must have a specific form. In particular, the pointwise mutual information $i(t:s)$ is given by
\begin{equation}
i(t:s) := \log\left(\frac{P(t|s)}{P(t)}\right)
\end{equation}
The mutual information $I(T:S)$ is then simply defined as the \textit{average} pointwise mutual information. Note that pointwise mututal information (in contrast to mutual information) can be both positive and negative. It essentially measures whether we are guided in the right or wrong direction with the respect to the actual target realization $t.$ If the conditional probability of $T=t$ given the observation of $S=s$ is larger than the unconditional (prior) probability of $T=t$, then we are guided in the right direction: The actual target realization is in fact $t$ and observing that $S=s$ makes us more likely to think so. Accordingly, in this case the pointwise mutual information is \textit{positive}. On the other hand, if the conditional probability of $T=t$ given the observation of $S=s$ is smaller than the unconditional (prior) probability of $T=t$, then we are guided in the wrong direction: Observing $S=s$ makes us less likely to guess the correct target value. In this case the pointwise mutual information is \textit{negative}. The joint pointwise mutual information of source realizations $s_1,\ldots,s_n$ about the target realization is defined in just the same way:
\begin{equation}
i(t:s_1,\ldots,s_n) := \log\left(\frac{P(t|s_1,\ldots,s_n)}{P(t)}\right)
\end{equation}

The idea is now to perform the entire partial information decomposition on the pointwise level, i.e. to decompose the pointwise joint mutual information $i(t:s_1,\ldots,s_n)$ that the source realizations provide about the target realization \cite{finn2018pointwise}. This leads to \textit{pointwise atoms} $\pi_{s_1,\ldots,s_n,t}$ (in the following we will generally drop the subscript). Crucially, we are only changing the quantity to be decomposed from $I(T:S_1,\ldots,S_n)$ to $i(t:s_1,\ldots,s_n)$. Otherwise, the idea is completely analogous to what we have discussed in  \S\ref{sec:parthood} (simply replace $I$ by $i$ and $\Pi$ by $\pi$): the goal is to decompose the pointwise mutual information into information atoms that are characterized by their parthood relations to the pointwise mutual information provided by the different possible collections of source realizations. These atoms have to stand in the appropriate relationship to \textit{pointwise redundancy}: the pointwise redundancy  $i_\cap(t:\mathbf{a}_1,\ldots,\mathbf{a}_m)$ is the sum of all pointwise atoms $\pi(f)$ that are part of the information provided by \textit{each} collection of source realizations $\mathbf{a}_i$. By exactly the same argument as described in \S\ref{sec:parthood}ciii, there is a unique solution for the pointwise atoms once a measure of pointwise redundancy $i(t:\alpha)$ is fixed for all antichains $\alpha = \{\mathbf{a}_1,\ldots,\mathbf{a}_m\}$. The \textit{variable-level} atoms $\Pi$ are then defined as the \textit{average} of the corresponding pointwise atoms:
\begin{equation}\label{eq:average_atoms}
\Pi(f) = \sum\limits_{s_1,\ldots,s_n,t} P(s_1,\ldots,s_n,t) \pi_{s_1,\ldots,s_n}(f)
\end{equation}
We are now left with defining the pointwise redundancy $i_\cap$ among collections of source realizations. As noted above this is a much easier problem than coming up with a measure of redundancy among collections of entire source variables. In the next section, we show how the pointwise redundancy of multiple collections of source realizations can be expressed as the information provided by a particular \textit{logical statement} about these realizations.

\subsection{Defining pointwise redundancy in terms of logical statements}
The language of formal logic allows us to form statements about the source realizations. In particular, we will consider statements made up out of the following ingredients:
\begin{enumerate}
    \item $n$ basic statements of the form $S_i = s_i$, i.e. ``Source $S_i$ has taken on value $s_i$''
    \item the logical connectives $\wedge$ (and),  $\vee$ (or), $\neg$ (not), $\rightarrow$ (if, then)
    \item brackets ),(
\end{enumerate}

In this way, we may form statements such as ${S_1 = s_1 \wedge S_2 = s_2}$ (``Source $S_1$ has taken on value $s_1$ and source $S_2$ has taken on value $s_2$'') or ${S_1 = s_1 \vee (S_2 = s_2 \wedge S_3 = s_3 )}$ (``Either source $S_1$ has taken on value $s_1$ or source $S_2$ has taken on value $s_2$ and source $S_3$ has taken on value $s_3$''). Now we may ask: What is the information provided by the truth of such statements about the target realization $t$? Classical information theory allows us to quantify this information as a pointwise mutual information: Let $A$ be any statement of the form just described, then the information  $i(t:A)$ provided by the truth of this statement is
\begin{equation}
i(t:A) := i(t:\mathbb{I}_{A}=1) = \log\left(\frac{P(t| \text{$A$ is true})}{P(t)}\right)
\end{equation}
where $\mathbb{I}_{A}$ is the \textit{indicator random variable} of the event that the statement $A$ is true, i.e. $\mathbb{I}_{A} = 1$ if the event occurred and $\mathbb{I}_{A} = 0$ if it did not. The interpretation of this information is that it measures whether and to what degree we are guided in the right or wrong direction with respect to the actual target value once we learn that statement $A$ is true.

Note that according to this definition the pointwise mutual information provided by a collection of source realizations $i(t:\mathbf{a})$ is the information provided by the truth of the \textit{conjunction} $\bigwedge_{i\in \mathbf{a}} S_i = s_i$:
\begin{equation}
i(t:\mathbf{a}) = i\left(t:\bigwedge_{i\in \mathbf{a}} S_i = s_i\right)
\end{equation}
Therefore, the information redundantly provided by collections of source realizations $\mathbf{a}_1,\ldots,\mathbf{a}_m$ is the information redundantly provided by the truth of the corresponding conjunctions. Now, what is this information? We propose that in general the following principle describes redundancy among statements:
\begin{principle}
The information redundantly provided by the truth of the statements $A_1,\ldots,A_m$ is the information provided by the truth of their \textit{disjunction} $A_1 \vee\ldots\vee A_m$.
\end{principle}
There are two motivations for this principle: First, the logical inferences to be drawn from the disjunction $A \vee B$ are precisely the inferences that can be drawn \textit{redundantly} from both $A$ and $B.$ If some conclusion $C$ logically follows from both $A$ and $B,$ then it also follows from $A\vee B$. Conversely, any conclusion $C$ that follows from the disjunction $A\vee B$ follows from both A and B. Formally,
\begin{align}
A\vee B \models C \Leftrightarrow A \models C \text{ and } B \models C
\end{align}
where $\models$ denotes logical implication. The second motivation again invokes the idea of parthood relationships: \textit{If some statement $C$ is logically weaker than a statement $A,$ then the information provided by $C$ should be part of the information provided by $A$}. For instance, the information provided by the statement ${S_1 = s_1}$ has to be part of the information provided by the statement ${S_1 = s_1 \wedge S_2 = s_2}$. This idea is illustrated in the information diagram on the left side in Figure \ref{fig:log_mon_and_red_arrow}.

\begin{figure}
    \centering
	\includegraphics[width=0.5\textwidth]{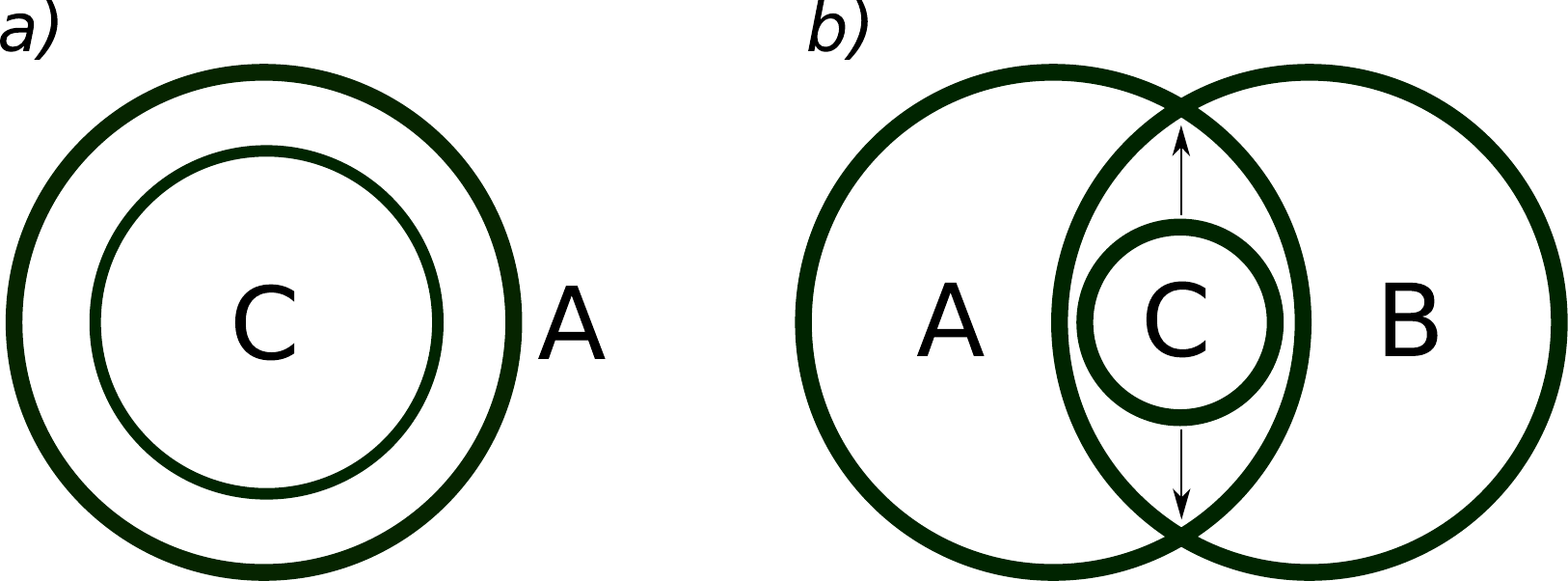}
	\caption{(a) Information diagram depicting the information provided by statements $A$  and $C$. If statement C is logically weaker than statement $A,$ i.e. if $C$ is implied by $A,$ then the information provided by $C$ has to be part of the information provided by $A.$ (b) Information diagram depicting the information provided by statements $A$, $B$, and $C$. $C$ is assumed to be logically weaker than both $A$ and $B.$ Thus it has to be part of the information provided by $A$ and also part of the information provided by $B.$ Accordingly, it is contained in the ``overlap'', i.e.~the redundant information of $A$ and $B.$ In order to obtain the entire redundant information statement $C$ has to be ``maximized'', i.e.~it has to be chosen as the strongest statement weaker than both $A$ and $B$ (this is indicated by the arrows).}
	\label{fig:log_mon_and_red_arrow}
\end{figure}

Now, this idea implies that if a statement $C$ is weaker than \textit{both} $A$ and $B,$ then the information provided by $C$ is part of the information carried by $A$ and also part of the information carried by $B.$ But this means that the information provided by $C$ is part of the \textit{redundant information } of $A$ and $B.$ In order to obtain the \textit{entire} redundant information, the statement $C$ should therefore be chosen as the \textit{strongest} statement logically weaker than both $A$ and $B$ (see right side of Figure \ref{fig:log_mon_and_red_arrow}). But this statement is the disjunction $A \vee B$ (or any equivalent statement).

Based on these ideas we can now finally formulate our proposal for a measure of pointwise redundancy $i_\cap(t:\mathbf{a}_1,\ldots,\mathbf{a}_m)$. We noted above that the information redundantly provided by collections of realizations $\mathbf{a}_1,\ldots,\mathbf{a}_m$ is the information redundantly provided by the conjunctions $\bigwedge_{i\in \mathbf{a}_j} S_i = s_i$. And by the arguments just presented this is the information provided by the \textit{disjunction of these conjunctions}. We denote the function that measures pointwise redundant information in this way by $i_\cap^{\text{sx}}$ (for reasons that will be explained shortly). It is formally defined as:

\begin{equation}
i_\cap^{\text{sx}}(t:\mathbf{a}_1,\ldots,\mathbf{a}_m) := i\left(t:\bigvee_{j=1}^m\bigwedge_{i\in \mathbf{a}_j} S_i = s_i\right)
\end{equation}

Recall that by definition this is the pointwise mutual information provided by the truth of the statement in question. Hence, it measures whether and to what degree we are guided in the right or wrong direction with respect to the actual target value $t$ once we learn that the statement is true. 

We have now arrived at a \textit{complete} solution to the partial information decomposition problem: Given the measure $i_\cap^{\text{sx}}$ we may carry out the Moebius-Inversion
\begin{equation}\label{eq:moebius_inversion_sx_pw}
i_\cap^{\text{sx}}(t:f) = \sum\limits_{g \sqsubseteq f} \pi^{\text{sx}}(f)
\end{equation} 
in order to obtain the pointwise atoms $\pi^{\text{sx}}$. This has to be done for \textit{each} realization $s_1,\ldots,s_n,t$. The obtained values can then be averaged as per Equation \eqref{eq:average_atoms} to obtain the variable-level atoms $\Pi^{\text{sx}}$.

As shown in \cite{makkeh2020isx}, the measure $i_\cap^{\text{sx}}$ can also be motivated in terms of the notion of  \textit{shared exclusions} (hence the superscript ``sx''). The underlying idea is that redundant information is linked to possibilities (i.e. points in sample space) that are redundantly excluded by multiple source realizations. We argue that the fact that the measure $i_\cap^{\text{sx}}$ can be derived in these two independent ways provides further support for its validity. We offer a freely accessible implementation of the $i_\cap^{sx}$ PID as part of the IDTxl toolbox \cite{wollstadt2018idtxl}. Worked examples of its computation and details on the computational complexity can be found in \cite{makkeh2020isx}.

In the following section, we show that the value of formal logic within the theory of partial information decomposition goes far beyond helping us to define a measure of pointwise redundant information. In fact, similar to the lattices of parthood distributions and antichains, there is a lattice of logical statements that can equally be used as the basic mathematical structure of partial information decomposition. This lattice is particularly useful because the ordering relationship turns out to be very simple and well-understood: the relation of logical implication. We will show that this perspective also offers an independent starting point for the development of PID theory.


\section{The logical perspective}\label{sec:logic}

\subsection{Logic Lattices}
The considerations of the previous section identified the information redundantly provided by collections $\mathbf{a}_1,\ldots,\mathbf{a}_m$ with the information provided by a particular logical statement: a disjunction of conjunctions of basic statements of the form $S_i = s_i$. This has an interesting implication: there is a one-to-one mapping between antichains $\alpha$ and logical statements. Let us now look at this situation a bit more abstractly by replacing the concrete statements $S_i = s_i$ with \textit{propositional variables} $\phi_1,\ldots,\phi_n$. Together with the logical connectives $\neg, \vee, \wedge, \rightarrow$ (plus brackets) these form a language of propositional logic \cite{smullyan1995first}. We will denote this language by $\mathbb{L}$. We may now formally introduce a mapping $\Psi$ from the set of antichains $\mathcal{A}$ into $\mathbb{L}$ via
\begin{equation}
\Psi:\mathcal{A} \rightarrow \mathbb{L}, \hspace{0.1cm} \text{ where } \alpha \rightarrowtail \tilde{\alpha} := \bigvee_{a \in \alpha} \bigwedge_{i\in a} \phi_i
\end{equation}
In other words, $\alpha$ is mapped to a statement by first conjoining the $\phi_i$ corresponding to indices \textit{within} each $\mathbf{a}_i$ and then disjoining these conjunctions. For instance, the antichain $\{\{1,2\},\{2,3\}\}$ will be associated with the statement $(\phi_1 \wedge \phi_2) \vee (\phi_2 \wedge \phi_3)$. Note that if we interpret the propositional variables $\phi_i$ as ``source $S_i$ has taken on value $s_i$'', then this is of course precisely the mapping of an antichain to the statement providing the redundant information (in the sense of $i_\cap^{sx}$) associated with that antichain. \footnote{There is a slight ambiguity in the definition of $\Psi$ since the \textit{order} of the conjunctions $\bigwedge_{i\in a} \phi_i$ and statements $\phi_i$ is not specified. This problem can be solved, however, by choosing any enumeration of the elements $\mathbf{a}$ of the powerset of $\{1,\ldots,n\}$ and ordering the conjunctions $\bigwedge_{i \in \mathbf{a}} \phi_i$ accordingly. The propositional variables $\phi_i$ within the conjunctions may simply be ordered by ascending order of their indices.}

The range $\mathcal{L} \subseteq \mathbb{L}$ of $\Psi$ is \textit{set of all disjunctions of logically independent conjunctions of pairwise distinct propositional variables}. The logical independence of the conjunctions is the logical counterpart of the antichain property. The ``pairwise distinct'' condition ensures that the same atomic statement does not occur multiple times in any conjunction. The set $\mathcal{L}$ can now be equipped with the relationship of logical implication $\models$ in order to obtain a new structure $(\mathcal{L},\leftmodels)$ which we will show to be isomorphic to the lattices of antichains and parthood distributions. Here $\models$ means ``implies'' and $\leftmodels$ means ``is implied by''.

Based on these concepts, the following theorem expresses the isomorphism of $(\mathcal{L},\leftmodels)$ to the lattices of antichains and parthood distributions:
\begin{thm}
	For all $n\in \mathbb{N}$: $(\mathcal{L}_n,\leftmodels)$ is isomorphic to $(\mathcal{A}_n, \preceq)$ and $(\mathcal{B}_n, \sqsubseteq)$
\end{thm}
\begin{proof}
See Electronic Appendix \ref{app:proof_isomorph}.
\end{proof}
\begin{cor}\label{cor:pos_lat}
		For all $n\in \mathbb{N}$: $(\mathcal{L}_n,\leftmodels)$ is a poset and specifically a lattice.
\end{cor}

\begin{figure}
    \centering
	\includegraphics[width=0.9\textwidth]{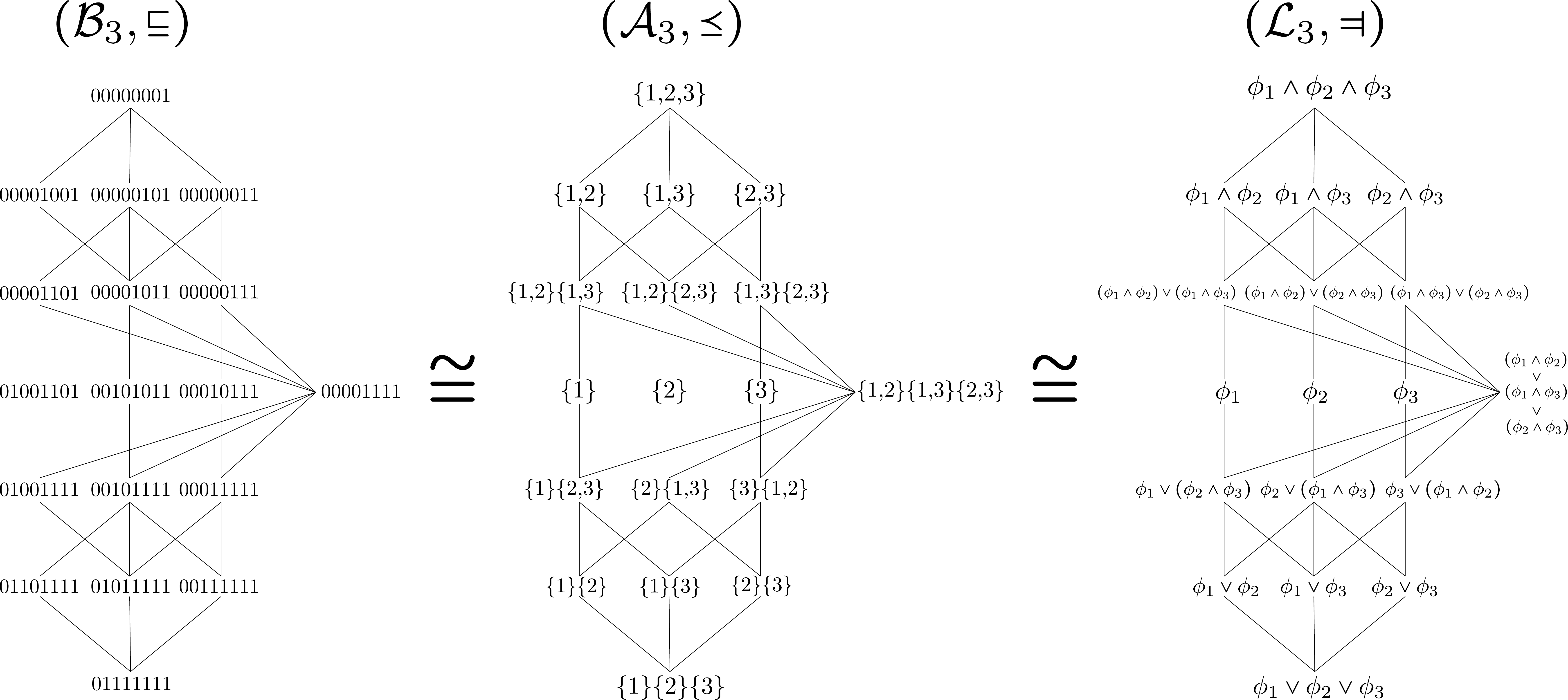}
	\caption{The three isomorphic worlds of partial information decomposition: parthood distributions, antichains, and logical statements.}
	\label{fig:3_lattice_mon_bool_antichain_statement}
\end{figure}
In this way the logical perspective is put on equal footing with the parthood perspective and "antichain" perspective described by Williams and Beer \cite{williams2010nonnegative}. They are in fact three equivalent ways to describe the mathematical structure underlying partial information decomposition. These three ``worlds'' of PID are illustrated in Figure \ref{fig:3_lattice_mon_bool_antichain_statement} for the case of three information sources. 

Intuitively, the logic lattice can be understood as a hierarchy of logical constraints describing how (i.e. via which collections of sources) information about the target may be accessed. The information atom associated with a node $\tilde{\alpha}$ in the logic lattice is \textit{exactly} the information about the target that can be accessed in the way described by the constraint $\tilde{\alpha}$. For example, the information shared by all sources $\Pi(\{1\},\{2\},\{3\})$ is to be found at the very bottom of the logic lattice because access to this information is constrained in the least possible way: the shared information can be accessed via \textit{any} source (i.e. via source 1 \textit{or} source 2 \textit{or} source 3). By monotonicity, the shared information is of course also accessible via any \textit{collection} of sources so that in total there are seven ways to access it (one per collection). By contrast, the all-way synergy $\Pi(\{1,2,3\})$ is located at the very top of the logic lattice because access to it is most heavily constrained: the synergy can only be accessed if all sources are known at the same time. Hence, there is only a single way (collection) to access it. All other atoms are to be found in between these two extremes. For instance, the atom corresponding to the constraint $\phi_1 \vee (\phi_2 \wedge \phi_3)$ is exactly the information that can be accessed either via source 1 or via sources 2 and 3 jointly (and of course via any superset of these collections by monotonicity) \textit{but not in any other way} (i.e. not via sources 2 or 3 individually). So in total there are five ways to access it corresponding to the collections $\{1\}, \{1,2\}, \{1,3\}, \{2,3\},\{1,2,3\}$. In general, the atoms on the k-th level of the logic lattice (starting to count at the top) are precisely the atoms that can be accessed via k collections of sources (compare this to the very similar insight in \S\ref{sec:parthood}ciii).

Finally, one may also associate a redundant information term with each node in the logic lattice by interpreting the statements as \textit{sufficient conditions} for access instead of \textit{constraints}, i.e. sufficient and necessary conditions, on access.  For instance, the redundancy associated with the statement  $\phi_1 \wedge \phi_2 \wedge \phi_3$ would be all information for which joint knowledge of all three sources is sufficient. But this is of course \textit{all} information contained in the sources, i.e. the entire joint mutual information. By contrast, the information atom associated with the same statement is the information for which joint knowledge of all three sources is not only sufficient but also necessary, i.e. it cannot be obtained via any other collection of sources. Or put generally: while the redundancy is the information we obtain \textit{if} we have knowledge of certain collections of sources, the information atom is the information we obtain \textit{if and only if } we have such knowledge. Defined in this way the redundant information of a lattice node is again the sum of atoms associated with nodes below and including it.

In this way the logical perspective can be used as an independent starting point to develop PID theory. Instead of characterizing atoms by their defining parthood relations one might equally characterize them by their defining access constraints and relate them to the notion of redundant information in the way just described. This is summarized in the following Core Principle:
\begin{principle}\label{prc:core_principle_logic}
Each atom of information is characterized by a logical constraint describing via which collections of sources it can be accessed. The atom $\Pi(\ta)$ associated with constraint $\ta = \bigvee_{\bfa \in \alpha} \bigwedge_{i\in \bfa} \phi_i$ is exactly that part of the joint mutual information about the target that can be accessed if and only if we have knowledge of any one of the collections of sources $\bfa$.
\end{principle}

Now that we have fully introduced both the parthood and logical approaches to PID it is worth noting their key difference to the original "antichain" approach by Williams and Beer: whereas the parthood and logic approaches are looking at the problem from the perspective of the atoms and seek to describe their defining parthood relations / access constraints, the antichain based approach starts off by placing certain axioms on measures of redundant information leading to the insight that the definition of redundancy may be restricted to antichains. The atoms are then \textit{indirectly} introduced in terms of a Moebius-Inversion over the lattice of antichains.

The next section highlights an additional use of logic lattices, namely as a mathematical tool to analyse the structure of PID lattices.


\subsection{Using logic lattices as a mathematical tool to analyse the structure of PID lattices}
One advantage that logic lattices have over the lattices of antichains and parthood distributions is that their ordering relationship is particularly natural and well-understood: logical implication between statements. By contrast, the ordering relation $\preceq$ on the lattice of antichains only seems to have been studied in quite restricted order theoretic contexts so far. Furthermore, it is a purely technical concept that does not have a clear-cut counterpart in ordinary language. Because of the simplicity of its ordering relation, many important order theoretic concepts have a simple interpretation within the logic lattice. This makes it a useful tool to understand the structure of the lattice itself which in turn is relevant to the computation of information atoms. 

There is an interesting fact about the statements in $\mathcal{L}$ that will be useful in the following investigations:  they correspond to statements with monotonic truth-tables. The truth-table $T_{\tilde{\alpha}}: \mathcal{V} \rightarrow \{0,1\}$ of a statement $\tilde{\alpha}$ describes which models $V \in \mathcal{V}$ satisfy $\tilde{\alpha}$ (``make $\tilde{\alpha}$ true''), i.e.
\begin{equation}
T_{\tilde{\alpha}}(V) =
\begin{cases}
1 & \text{if } \models_V \tilde{\alpha}\\
0 & \text{otherwise }\\
\end{cases}
\end{equation} 
A truth-table $T$ shall be called \textit{monotonic} just in case $\forall i \in \{1,\ldots,n\}$
\begin{equation}
\small{(V(\phi_i) = 1 \rightarrow V^\prime(\phi_i) = 1) \Rightarrow \left(T(V) = 1 \rightarrow T(V^\prime) = 1\right)}
\end{equation}
In other words, suppose a statement $\tilde{\alpha}$ is satisfied by a valuation $V$. Now suppose further that a new valuation $V^\prime$ is constructed by flipping one or more zeros to one in $V$. Then $\tilde{\alpha}$ has to be satisfied by $V^\prime$ as well. Making some $\phi_i$ true that were previously false cannot make $\tilde{\alpha}$ false if it was previously true. With this terminology at hand the following proposition can be formulated:
\begin{prop} \label{prop:mon_tt}
	All $\tilde{\alpha} \in \mathcal{L}$ have monotonic truth-tables. Conversely, for all monotonic truth-tables T, there is exactly one $\tilde{\alpha} \in \mathcal{L}$ such that $T_{\tilde{\alpha}} = T$. In other words, the statements in $\mathcal{L}$ are, up to logical equivalence, exactly the statements of propositional logic with monotonic truth-tables.
\end{prop}
\begin{proof}
See Appendix \ref{app:proofs_of_props}
\end{proof}

Now, it was shown in \cite{finn2018pointwise} that the information atoms have a closed form solution in terms of the \textit{meets} of any subset of children of the corresponding node in the lattice. The \textit{meet} (infimum) and \textit{join} (supremum) operations, however, have quite straightforward interpretations on $(\mathcal{L}, \leftmodels)$: The meet of two statements $\tilde{\alpha}$ and $\tilde{\beta}$ is the strongest statement logically weaker than both $\ta$ and $\tb$. Similarly, the join is the weakest statement logically stronger than both $\ta$ and $\tb$. The meet is logically equivalent (though not identical) to the \textit{disjunction} of $\ta$ and $\tb$ while the join is logically equivalent (though not identical) to their \textit{conjunction}. The conjunction and disjunction of two elements of $\mathcal{L}$ do generally not lie in $\mathcal{L}$ because they do not necessarily have the appropriate form (disjunction of logically independent conjunctions). However, this can easily be remedied because both the disjunction and the conjunction of elements of $\mathcal{L}$ have monotonic truth-tables. Thus, by Proposition \ref{prop:mon_tt} there is a unique element in $\mathcal{L}$ with the same truth-table in both cases. These elements are therefore the meet and join. The detailed construction of meet and join operators is presented in Appendix \ref{app:proofs_of_props}.

Let us now turn to the notions of child and parent. A \textit{child} of a statement  $\tilde{\alpha} \in \mathcal{L}$ is a strongest statement strictly weaker than $\tilde{\alpha}$. Similarly, a \textit{parent} of $\tilde{\alpha}$ is a weakest statement strictly stronger than $\tilde{\alpha}$. The following three propositions provide, first, a characterization of children in terms of their truth tables, second, a lower bound on the number of children of a statement, and third, an algorithm to determine all children of a statement. Due to the isomophism of antichains, parthood distributions, and logical statements, the propositions can be utilized to study any of these three structures.
\begin{prop}[Characterization of Children]\label{prop:children_characterization} \label{prop:char_children}
	$\tilde{\gamma}\in \mathcal{L}$ is a direct child of $\tilde{\alpha} \in \mathcal{L}$ if and only if $\tilde{\gamma}$ is true in all cases in which $\tilde{\alpha}$ is true plus exactly one additional case, i.e. just in case $T_{\tilde{\alpha}}(V) = 1 \rightarrow T_{\tilde{\gamma}}(V) = 1$ and $!\exists V \in \mathcal{V}: T_{\tilde{\gamma}}(V) = 1 \wedge T_{\tilde{\alpha}}(V) = 0$.
\end{prop}
\begin{proof}
See Appendix \ref{app:proofs_of_props}
\end{proof}
\begin{prop}[Lower bound on number of children]\label{prop:lower_bound_children}
	Any $\alpha \in \mathcal{A}$ such that there is at least one $\mathbf{a} \in \alpha$ with $|\mathbf{a}| = k \geq 1$ has at least k children.
\end{prop}
\begin{proof}
See Appendix \ref{app:proofs_of_props}
\end{proof}
\begin{prop}[Algorithm to determine children] \label{prop:algorithm_children}
The children of a statement $\tilde{\alpha}$ can be determined via the following algorithm (for a pseudocode version see Appendix \ref{app:proofs_of_props}). It proceeds in three steps:

\begin{enumerate}
\item Set k to the maximal number of ones occurring in a valuation that does not satisfy $\tilde{\alpha}$.
\item For each valuation $V$ that does not satisfy $\tilde{\alpha}$ and contains k ones do the following:
\begin{enumerate}
\item Check if there is a valuation with k+1 ones that does not satisfy $\tilde{\alpha}$  and results from flipping one or multiple zeros in $V$ to one, i.e. a model $V^\prime$ such that $V(\phi_i) = 1 \rightarrow V^\prime(\phi_i) = 1$. If there is such a valuation, then skip step b). Otherwise, proceed.
\item Create a new monotonic truth-table by setting $V$ to one, otherwise leaving the truth-table of $\tilde{\alpha}$ unchanged. The statement corresponding to this truth-table is a child of $\tilde{\alpha}$.
\end{enumerate}
\item If $k>0$, decrease k by 1 and repeat Step 2. Otherwise, terminate.
\end{enumerate}
\end{prop}

\begin{proof}
See Appendix \ref{app:proofs_of_props}
\end{proof}
This concludes our discussion of the relationship between formal logic and PID. In the next section we return to the parthood side of our story. In particular, we will address an apparent arbitrariness in the argument presented in \S\ref{sec:parthood}c. Here we showed that the sizes of the atoms of information can be obtained once a measure of redundant information is specified. Now, one may ask of course: why redundant information? Couldn't the same purpose be achieved by utilizing some other informational quantity such as synergistic or unique information? We will now discuss how the parthood approach can help answering this question in a systematic way.

\section{Non-Redundancy based PIDs}\label{sec:discussion_non_red_PID}

Let us briefly revisit the structure of the argument in  \S\ref{sec:parthood}c. It involved three steps (presented in slightly different order above): First, based on the very concept of redundant information, we phrased a condition describing which atoms are part of which redundancies (Core Principle \ref{prc:parthood_criterion_red}). Secondly, we showed that this parthood criterion entails a number of contraints on the measure $I_\cap$. Finally, we showed that, as long as these constraints are satisfied, we obtain a unique solution for the atoms of information. There is actually a fourth step: We would have to check that the information decomposition satisfies the consistency equations relating atoms to mutual information terms (Equation \ref{eq:quant_relation_atoms_MI}). However, in the case of redundant information this condition is trivially satisfied due to the self-redundancy property. In other words, the consistency equations are themselves part of the system of equations used to solve for the information atoms.   

In order to obtain an information decomposition based on a quantity other than redundant information, lets call it $I^*(T:\bfa_1,\ldots,\bfa_m)$, we may use precisely the same scheme:

\begin{enumerate}
    \item Define a condition $\mathcal{C}(f:\bfa_1,\ldots,\bfa_m)$ on parthood distributions $f$ describing which atoms $\Pi(f)$ are part of $I^*(T:\bfa_1,\ldots,\bfa_m)$ for any given tuple of collections of sources $\bfa_1,\ldots,\bfa_m$. This leads to a system of equations:
    \begin{equation}
        I^*(T:\bfa_1,\ldots,\bfa_m) = \sum_{\mathcal{C}(f:\bfa_1,\ldots,\bfa_m)} \Pi(f)
    \end{equation}
    \item Analyse which constraints on $I^*(T:\bfa_1,\ldots,\bfa_m)$ (e.g. symmetry, idempotency, \dots) are implied by this relationship.
    \item Show that given a choice of $I^*(T:\bfa_1,\ldots,\bfa_m)$ that satisfies the constraints, a unique solution for all information atoms $\Pi(f)$ can be obtained.
    \item Show that the solution satisfies the consistency equation
    \eqref{eq:quant_relation_atoms_MI} relating information atoms and mutual information terms.
\end{enumerate}

Let us work through these steps in specific cases.

\subsection{Restricted Information PID}

Recall that the redundant information of multiple collections of sources is the information we obtain \textit{if} we have access to any of the collections. Similarly, we can define the information ``restricted by'' collections of sources $\bfa_1,\ldots,\bfa_m$ as any information we obtain \textit{only if} we have access to at least one of the collections.  For instance, assuming $n=2$, the information restricted by the first source consists of its unique information and its synergy with the second source. Both of these quantities can only be obtained if we have access to the first source.

Thus, in general the restricted information $I_\text{res}(T:\mathbf{a}_1,\ldots,\mathbf{a}_m)$ should consist of all the atoms that are \textit{only} part of the information carried by some of the $\mathbf{a}_i$ \textit{but not part of the information provided by any other collection of sources}. Thus the parthood condition $\mathcal{C}_\text{res}$ is given by
\begin{equation}
\mathcal{C}_\text{res}(f:\bfa_1,\ldots,\bfa_m) \Leftrightarrow (f(\mathbf{b})=1 \rightarrow \exists i: \mathbf{b} \supseteq \mathbf{a}_i)
\end{equation}
and we obtain the relation
\begin{equation}\label{eq:quant_relation_atoms_restr}
I_\text{res}(T:\mathbf{a}_1,\ldots,\mathbf{a}_m) = \sum\limits_{\mathcal{C}_\text{res}(f:\bfa_1,\ldots,\bfa_m) } \Pi(f)
\end{equation}
Just as in the case of redundant information, this relationship implies a number of invariance properties for $I_\text{res}$: it has to be symmetric, idempotent, and invariant under superset removal/addition allowing us again to restrict ourselves to the set of antichains. The analogue of the "self-redundancy" property is that the restricted information of a collection of singletons ${I_\text{res}(T:\{i_1\},\ldots,\{i_m\})}$ is equal to the  \textit{conditional mutual information provided by their union $\alpha_\cup  = \bigcup_{j=1}^m \{i_j\}$ conditioned on all other sources}. So if $\alpha = \{\{i_1\},\ldots,\{i_m\}\}$ is a collection of singletons, then:
\begin{equation}\label{eq:restr_info_cond_MI}
I_\text{res}(T:\alpha) = I\left(T:(S_i)_{i\in \alpha_\cup} | (S_j)_{j \in \alpha_\cup^C}\right)
\end{equation}
This can be established using the chain rule for mutual information as detailed in Appendix \ref{app:restr_info}. The next step is to show that we may obtain a unique solution for the information atoms once a measure of restricted information satisfying these conditions is given. This can be achieved in much the same way as for redundant information. The restricted information associated with an antichain $\alpha$ can be expressed as a sum of information atoms $\Pi(\beta)$ below and including $\alpha$ in a specific lattice of antichains $(\mathcal{A},\preceq^\prime)$. This lattice is simply the dual (inverted version) of the antichain lattice $(\mathcal{A},\preceq)$, i.e.
\begin{equation}
\alpha \preceq^\prime \beta \Leftrightarrow  \beta \preceq \alpha
\end{equation}
Accordingly, a unique solution is guaranteed via Moebius-Inversion of the relationship
\begin{equation}\label{eq:moebius_restr_info}
I_\text{res}(T:\alpha) = \sum_{\beta \preceq^\prime \alpha} \Pi_\text{res}(\alpha)
\end{equation}
As a final step we need to show that the resulting atoms stand in the appropriate relationships to mutual information terms. These relationships are given by the consistency equation \eqref{eq:quant_relation_atoms_MI}. Again using the chain rule it can be shown that this equation is equivalent to a condition relating conditional mutual information to the information atoms:
\begin{equation}
I(T:\mathbf{a}) = \sum\limits_{f(\mathbf{a})= 1} \Pi(f) \hspace{0.5cm}
\Leftrightarrow  \hspace{0.5cm}
I(T:\mathbf{a}|\mathbf{a}^C) = \sum\limits_{f(\mathbf{a}^C) = 0} \Pi(f)
\end{equation}
Now consider any collection of source indices ${\bfa = \{j_1,\ldots,j_m\}}$, then we obtain
\begin{align}
 I\left(T:\bfa| \bfa^C\right) &\stackrel{Eq. \eqref{eq:restr_info_cond_MI}}{=}   I_\text{res}(T:\{j_1\},\ldots,\{j_m\})  \\
&\stackrel{Eq. \eqref{eq:quant_relation_atoms_restr}}{=} \sum\limits_{f(\mathbf{b})=1 \rightarrow \exists i: \mathbf{b} \supseteq \{j_i\}} \Pi_\text{res}(f) \\
&= \sum\limits_{f(\bfa^C)=0} \Pi_\text{res}(f)
\end{align}
where the last equality follows because in the case of singletons the parthood condition $\mathcal{C}_{\text{res}}$ reduces to ${f(\alpha_\cup^C)=0}$. This establishes that the resulting atoms satisfy the consistency condition and we obtain a valid PID. 
In the following section we will use the same approach to analyse the question of whether a synergy based PID is possible.

\subsection{Synergy based PID}

Note that the restricted information of multiple collections of sources stands in a direct correspondence to a weak form of synergy which we will denote by $I_\text{ws}(T:\mathbf{a}_1,\ldots,\mathbf{a}_m)$. This quantity is to be understood as \textit{the information about the target we cannot obtain from any individual collection} $\bfa_i$. Accordingly, the parthood criterion is
\begin{equation}
\mathcal{C}_\text{ws}(f:\bfa_1,\ldots,\bfa_m) \Leftrightarrow ( \forall i \in \{1,\ldots,m\}: f(\bfa_i)=0 )
\end{equation}
But this information is of course the same as the information that we can get only if some \textit{other} collections are known (except subcollections of course), i.e. 
\begin{equation}
I_\text{ws}(T:\bfa_1,\ldots,\bfa_m) = I_\text{res}(T:(\mathbf{b} \mid \forall i\mathbf{b} \not\subseteq \bfa_i ))
\end{equation}
Consider the case of two sources: the information we cannot get from source 2 alone, $I_\text{ws}(T:\{2\})$, is the same as the information we can get only if the first source is known, ${I_\text{res}(T:\{1\})}$: unique information of source 1 plus synergistic information.

Due to this correspondence, the argument presented above can also be used to show that a consistent PID can be obtained by fixing a measure $I_\text{ws}$ of weak synergy. Once such a measure is given we can first translate it to the corresponding restricted information terms and then perform the Moebius inversion of Equation \eqref{eq:moebius_restr_info} (alternatively, the above argument could be redeveloped directly for $I_\text{ws}$ with minor modifications)

 Interestingly, if we associate with every antichain $\alpha$ in the lattice ${(\mathcal{A},\preceq)}$ the corresponding ${I_{\mathrm{ws}}(T:\beta)}$ (so that ${I_{\mathrm{res}}(T:\alpha) = I_{\mathrm{ws}}(T:\beta)}$), then the $\beta$ form an isomorphic lattice but with a different ordering (see Figure \ref{fig:redundancy_vs_constraint_lattice}). Just as the original antichain lattice this structure on the antichains has been introduced by Crampton and Loizou \cite{Crampton2000}. 
 
In the PID field a restricted version of this lattice (i.e. restricted to a certain subset of antichains) has been described by \cite{James2019} and \cite{Ay2019} under the name ``constraint lattice''. This terms is also appropriate in the present context: Intuitively, if we move up the constraint lattice we encounter information that satisfies more and more constraints. First, all of the information in the sources ($I_\text{ws}(T:\emptyset)$). This is the case of no constraints. Then all the information that is not contained in a particular individual source ($I_\text{ws}(T:\{1\})$ and $I_\text{ws}(T:\{2\})$). And finally the information that is not contained in any individual source ($I_\text{ws}(T:\{1\},\{2\})$) . 

Most recently, the full version of the lattice (i.e.~ defined on all antichains) has been utilized by \cite{Rosas2020synergy} to formulate a synergy centered information decomposition. They call the lattice \textit{extended constraint lattice} and define "synergy atoms" $S_{\partial}$ in terms of a Moebius-Inversion over it. The concept of synergy $S^{\boldsymbol{\alpha}}$  utilized in this approach closely resembles what we have called weak synergy. However, the decomposition is \textit{structurally different} from the type of decomposition discussed here and generally assumed in previous work on PID. Even though it leads to the same number of atoms, these atoms do not stand in the expected relationships to mutual information. For instance, in the 2-sources case, there is no pair of atoms that necessarily adds up to the mutual information provided by the first source and no such pair of atoms for the second source. The consistency equation (\ref{eq:quant_relation_atoms_MI}) is not satisfied (except for the full set of sources). This means that synergy atoms $S_{\partial}$ are not directly comparable to standard PID atoms $\Pi$. They represent different types of information.

\begin{figure}
\centering
\includegraphics[width=0.4\textwidth]{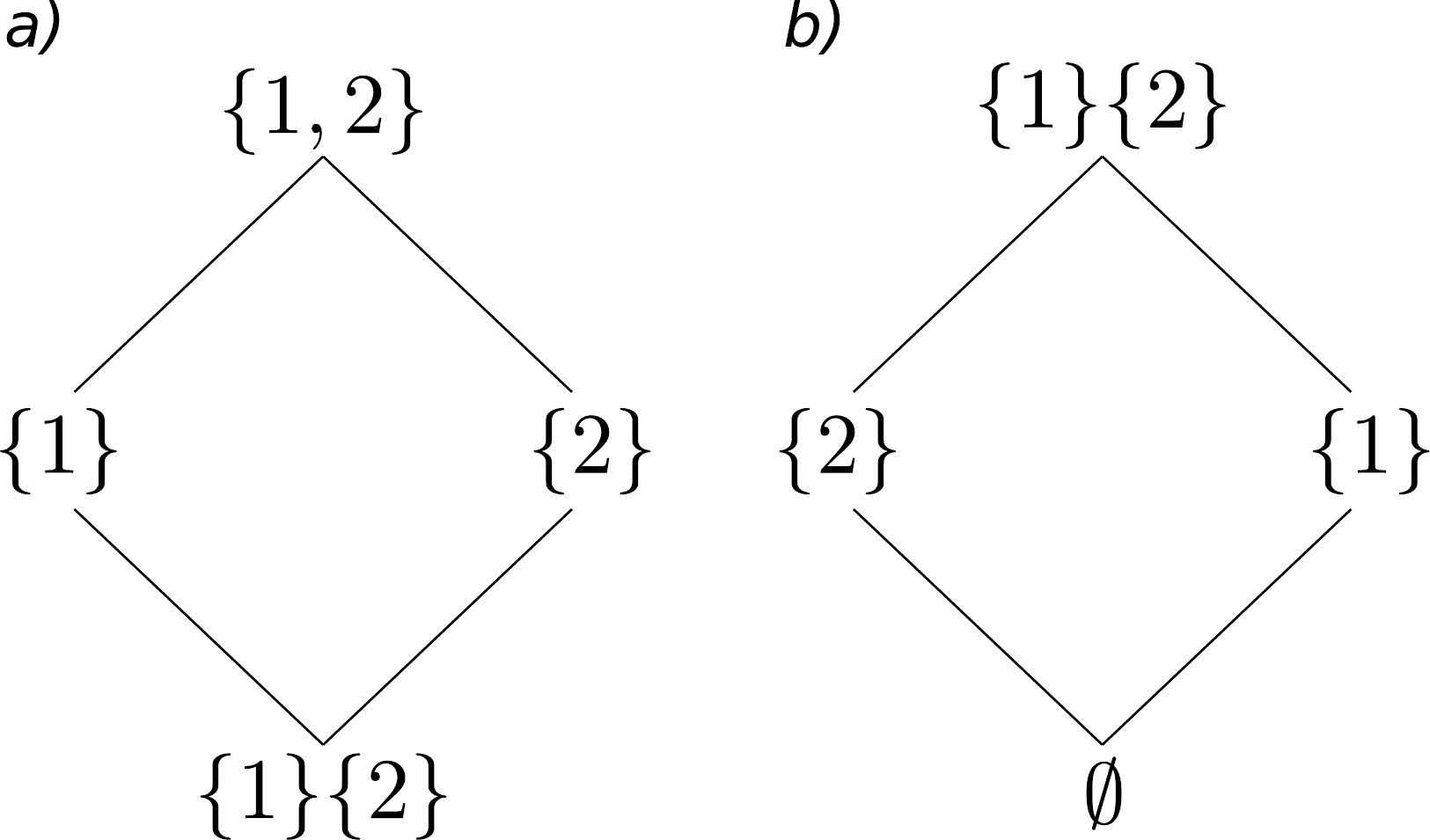}
\caption{(a) antichain lattice $(\mathcal{A}_2,\preceq)$ for two sources. Summing up the atoms \textit{above} and including a node yields the restricted information of that node. (b) extended constraint lattice for two sources. The weak synergy associated with a node in the extended constraint lattice is the sum of atoms above and including the corresponding node in the left lattice. Note that following a widespread convention we left out the outer curly brackets around the antichains.}
\label{fig:redundancy_vs_constraint_lattice}
\end{figure}

Let us now move towards stronger concepts of synergistic information. The reason for the term "weak" synergy is that a key ingredient of synergy seems to be missing in its definition: intuitively, the synergy of multiple sources is the information that cannot be obtained from any individual source but that becomes "visible" once we know all the sources at the same time. However, the definition of weak synergy only comprises the first part of this idea. The weak synergy $I_\text{ws}(T:\bfa_1,\ldots,\bfa_m)$ also contains parts that do not become visible even if we have access to \textit{all} $\bfa_i$. For instance, given $n=3$, the weak synergy $I_\text{ws}(T:\{1\},\{2\})$ also contains the unique information of the third source $\Pi(\{3\})$ because this quantity is accessible from neither the first nor the second source.

So let us add this missing ingredient by strengthening the parthood criterion:
{\small
\begin{equation}
\mathcal{C}_\text{ms}(f:\bfa_1,\ldots,\bfa_m) \!\Leftrightarrow\! ( \forall i \in \{1,\ldots,m\}: f(\bfa_i)=0 \!\text{ \& }\! f(\alpha_\cup) = 1)
\end{equation}
}
We obtain a moderate type of synergy we denote by $I_\text{ms}(T:\bfa_1,\ldots\bfa_m)$. It has a nice geometrical interpretation: in an information diagram it corresponds to all atoms \textit{outside} of all areas associated with the mutual information carried by some $\bfa_i$ but \textit{inside} the area associated with the mutual information carried by the union of the $\bfa_i$ (see Figure \ref{fig:moderate_synergy}). Furthermore, we can immediately see that the parthood condition cannot be satisfied for individual collections $\bfa$ (it demands $f(\bfa)=0$ and $f(\bfa)=1$ at the same time). This makes intuitive sense because the synergy of an individual collection appears to be an ill-defined concept: at least two things have to come together for there to be synergy. We will get back to the case of individual collections below.

\begin{figure}
\centering
\includegraphics[width=0.55\textwidth]{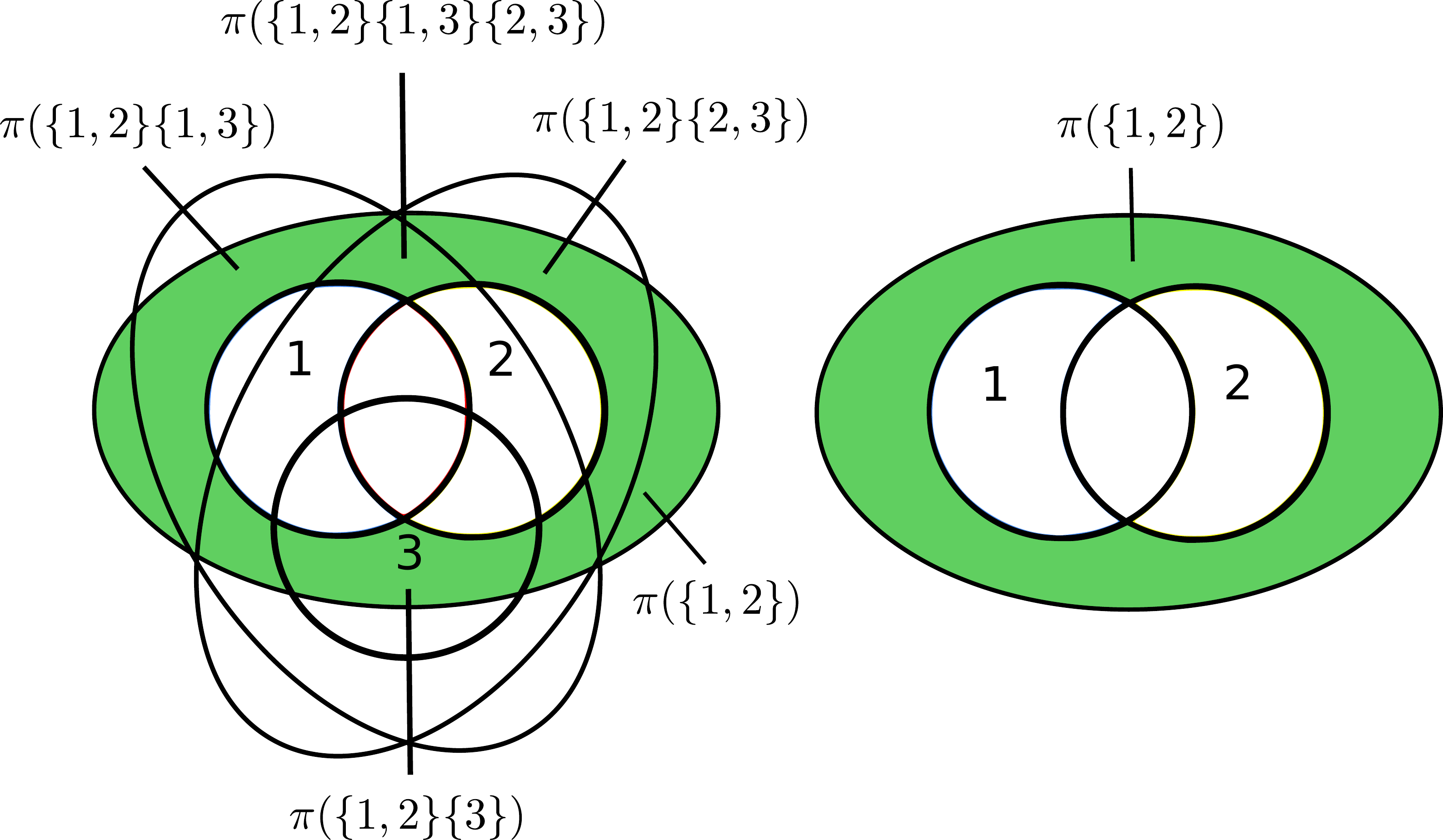}
\caption{Geometrical interpretation of moderate synergy $I_\text{ms}(T:\{1\},\{2\})$ for 2 and 3 sources.}
\label{fig:moderate_synergy}
\end{figure}
Let us first see what properties are implied by $\mathcal{C}_\text{ms}$. It can readily be shown that $I_\text{ms}$ is symmetric, idempontent, and invariant under \textit{subset} removal. This again allows us to restrict the domain of  $I_\text{ms}$ to the antichains. Additionally,  $I_\text{ms}$ satisfies the following condition:
\begin{equation}
\text{If } \exists i: \alpha_\cup = \bfa_i, \text{ then } I_\text{ms}(T:\alpha) = 0 \textbf{ (zero condition)}
\end{equation}
This property says that whenever the union of the collection happens to be equal to one of collections then the moderate synergy must be zero. This is in particular the case for the moderate "self-synergy" of a single collection. On first sight this raises a problem since the synergy equations associated with individual collections become trivial ($0=0$) and do not impose any constraints on the atoms. This situation can be remedied, however, by noting that these missing constraints are provided by the consistency equations relating the atoms to mutual information / conditional mutual information. In this way a unique solution for the atoms is indeed guaranteed (one could also axiomatically set the ``self-synergies'' to the respective conditional mutual information terms). The proof of this statement is given in Appendix \ref{app:restr_info}.  

An instructive fact about the moderate synergy based PID is that the underlying system of equations does not have the structure of a Moebius-Inversion over a lattice: there is no arrangement of atoms into a lattice such that each $I_\text{ms}(T:\alpha)$ turns out to be the sum of atoms below and including a particular lattice node. The reason is that any finite lattice always has a unique least element. In other words, some atom would have to appear at the very bottom of the lattice and would therefore be contained in all synergy terms. However, in the case of moderate synergy, there is no such atom for $n\geq 3$. The only viable candidate would be the overall synergy $\Pi(\{1,\ldots,n\})$. But due to the condition that the synergistic information has to become visible if we know all collections in question, this atom is not contained e.g. in $I_\text{ms}(T:\{1\},\{2\})$. 

Now one may wonder if the concept of synergy can be strengthened even further by demanding that the synergistic information should not be accessible from the \textit{union of any proper subset} of the collections in question. For instance, the synergistic information $I_\text{syn}(T:\{1\}\{2\}\{3\})$ of sources 1, 2, and 3 should not be accessible from the collections $\{1,2\}$, $\{1,3\}$, or $\{2,3\}$.  We have to know \textit{all three sources} to get access to their synergy. Thus, we may add this third constraint to obtain a strong notion of synergy we denote by $I_\text{syn}(T:\bfa_1,\ldots,\bfa_m)$. An atom $\Pi(f)$ should satisfy the corresponding parthood condition $\mathcal{C}_\text{syn}(f:\bfa_1,\ldots,\bfa_m)$ just in case 
{\small
\begin{enumerate}
    \item $f\left(\bigcup_{i=1}^m \mathbf{a}_i\right) = 1$ 
    \item $\forall i \in \{1,\ldots,m\}: f(\bfa_i) = 0$
    \item ${\forall J\subset \{1,\ldots,m\}, |J|\geq 2: \bigcup_{j \in J} \mathbf{a}_j \neq \bigcup_{i=1}^m \mathbf{a}_i \rightarrow f\left(\bigcup_{j \in J} \mathbf{a}_j\right) = 0}$
\end{enumerate}
}
The last condition is phrased as a conditional because the union of a proper subset of collection might happen to be equal to the union of all collections in question. Consider the case of three sources and the synergy $I_\text{syn}(T:\{1,2\}\{1,3\}\{2,3\})$. In this case the union of a proper subset of these collections, for instance $\{1,2\} \cup \{1,3\}$, happens to be equal to the union of all $\bfa_i$.

Unfortunately, we do not obtain enough linearly independent equations to uniquely determine the atoms of information. This can be shown using the example of three sources. According to the parthood criterion, $I_\text{syn}(T:\{1\}\{2\}\{3\}) = \Pi(\{1,2,3\})$. But also $I_\text{syn}(T:\{1,2\}\{1,3\}\{2,3\}) = \Pi(\{1,2,3\})$. This means that we do not obtain independent equations for each antichain. Or in linear algebras terms: our coefficient matrix will have two linearly dependent (actually identical) rows. Thus, a measure of strong synergy as described by $\mathcal{C}_\text{syn}$ cannot induce a unique PID.

\subsection{Unique information PID}
Let us briefly discuss the last obvious candidate quantity for determining the PID atoms: unique information \cite{bertschinger2014quantifying}. The appropriate parthood criterion for a measure of unique information $I_\mathrm{unq}$ seems straightforward in the case of individual collections $\mathbf{a}$: It should consist of all atoms that are part of the information provided by the collection $\mathbf{a}$ but not part of the information provided by any other collection. This is what makes this information ``unique'' to the collection. Since there is always just one such atom this means that $I_\mathrm{unq}(T:\mathbf{a}) =  \Pi(\mathbf{a})$. For instance, $I_\mathrm{unq}(T:\{1\}) =  \Pi(\{1\})$, as expected. However, defining $I_\mathrm{unq}$ only for individual collections does not yield enough equations to solve for the atoms. We need one equation per antichain / parthood distribution, and hence, some notion of the unique information associated with \textit{multiple} collections $\mathbf{a}_1,\ldots,\mathbf{a}_m$. This is a trickier question. What does it mean for information to be unique to these collections? Certainly, uniqueness demands that this information should not be contained in any \textit{other} collection. But what about the collections $\mathbf{a}_1,\ldots,\mathbf{a}_m$ themselves? It seems that the appropriate condition is that the unique information should consist of  atoms that are contained in \textit{all} of these collections. This idea aligns well with ordinary language: for instance, saying that a certain protein is unique to sheep and goats means that this protein is found in \textit{both sheep and goats and nowhere else}. Using this idea, the parthood criterion becomes
\begin{equation}
\mathcal{C}_\mathrm{unq}(f:\mathbf{a}_1,\ldots,\mathbf{a}_m) \Leftrightarrow \left(f(\mathbf{a}) = 1 \leftrightarrow \exists i: \mathbf{a} \supseteq \mathbf{a}_i \right)
\end{equation} 
However, this condition simply defines the atom $\Pi(\mathbf{a}_1,\ldots,\mathbf{a}_m)$ making the unique information based PID possible but maybe not very helpful: it just amounts to defining all the atoms separately because $I_\mathrm{unq}(T:\alpha) = \Pi(\alpha)$ for all antichains $\alpha$.

\section{Parthood descriptions vs. quantitative descriptions}\label{sec:discussion_parthood_vs_quant}
Before concluding we would like to briefly point out an issue that arises quite naturally when thinking about information theory from a parthood perspective and that merits a few remarks: throughout this paper we have drawn a distinction between \textit{parthood} relationships and \textit{quantitative} relationships between information contributions. In particular, Core Principles \ref{prc:core_principle_1} and \ref{prc:parthood_criterion_red} express parthood relationships between information atoms on the one hand and mutual information / redundant information on the other. Core Principle \ref{prc:wholes_are_sums_of_parts} by contrast describes the quantitative relationship between any information contribution and the parts it consists of. It is crucial to draw this distinction because these principles are logically independent. Consider the case of two sources: In this case, one could agree that the joint mutual information should consist of four parts while disagreeing that it should be the \textit{sum} of these parts. On other hand, one could agree that the joint mutual information should be the sum of its parts but disagree that it consists of four parts. 

The distinction between parthood relations and quantitative relations is also important in the argument that the redundant information provided by multiple statements is the information carried by the truth of their disjunction. One of the two motivations for this idea was based on the principle that the information provided by a statement $A$ is always \textit{part of} the information provided by any stronger statement $B.$ This does not mean however, that statement $A$ necessarily provides \textit{quantitatively} less information than $B$ (i.e. \textit{less bits} of information). In fact, this latter principle would contradict classical information theory. Here is why: suppose the pointwise mutual information ${i(t:s) = i(t:S=s)}$ is negative. Now, consider any \textit{tautology} such as ${S=s \vee \neg (S=s)}$. Certainly, this statement is \textit{logically weaker} than ${S=s}$ because a tautology is implied by any other statements. Furthermore, the probability of the tautology being true is equal to 1. Therefore, the information ${i(t:S=s \vee \neg (S=s))}$ provided by it is equal to 0. But this means ${i(t:S=s) < i(t:S=s \vee \neg (S=s))}$ even though ${S=s \vee \neg (S=s) \leftmodels S=s}$.

Nonetheless, there certainly is a sense in which a stronger statement $B$ provides ``more'' information than a weaker statement $A$: the information provided by $A$ is \textit{part of} the information provided by $B.$ If we know $B$ is true than we can by assumption infer that $A$ is true, and hence, we have access to all the information provided by $A.$ The fact that the stronger statement $B$ may nonetheless provide less bits of information can be explained in terms of misinformation: If we know $B$ is true, then we obtain all the information carried by $A$ \textit{plus some additional information}. If it happens that this surplus information is misinformative, i.e. negative, then quantitatively $B$ will provide less information than $A.$ This idea is illustrated in Figure \ref{fig:logical_monotonocity_principle_misinformative}.

\begin{figure}
\centering
\includegraphics[width=0.45\textwidth]{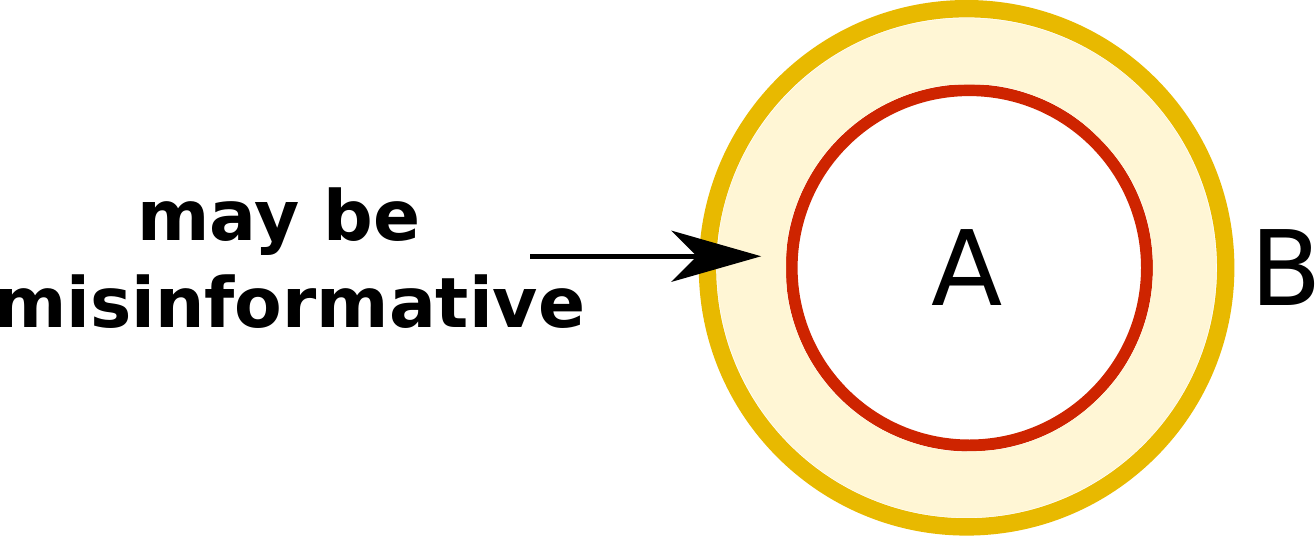}
\caption{Illustration of the idea that the information provided by a logically weaker statement A is always \textit{part of} the information of a stronger statement B, even though the latter may provide \textit{less bits} of information. This phenomenon can be explained in terms of the misinformative, i.e.~negative, contribution of the surplus information provided by B (the shaded ring).}
\label{fig:logical_monotonocity_principle_misinformative}
\end{figure}

Importantly, the possible negativity and non-monotonicity of  $i_\cap^\text{sx}$ as well as the potential negativity of $\pi^\text{sx}$ can be \textit{completely} explained in terms of misinformative contributions in the following sense: it is possible \cite{finn2018probability} to uniquely separate $i_\cap^\text{sx}$ into an informative part $i_\cap^\text{sx+}$ and a misinformative part $i_\cap^\text{sx-}$ such that 
\begin{equation}
i_\cap^\text{sx}(t:\alpha) = i_\cap^\text{sx+}(t:\alpha) - i_\cap^\text{sx-}(t:\alpha)
\end{equation}
Now, each of these components can be shown to be non-negative and monotonically increasing over the lattice. Moreover, the induced informative and misinformative atoms $\pi^\text{sx+}$ and $\pi^\text{sx-}$ are non-negative as well \cite{makkeh2020isx}. In other words, once we seperate out informative and misinformative components any violations of non-negativity and monotonicity disappear. Hence, these violations can be fully accounted for in terms of misinformative contributions.

\section{Conclusion} \label{sec:conclusion}
In this paper we connected PID theory with ideas from mereology, i.e. the study of parthood relations, and formal logic. The main insights derived from these ideas are that the general structure of information decomposition as originally introduced by Williams and Beer~\cite{williams2010nonnegative} can be derived entirely from 1) parthood relations between information contributions and 2) in terms of a hierarchy of logical constraints on how information about the target can be accessed. In this way the theory is set up from the perspective of the atoms of information, i.e. the quantities we are ultimately interested in. The $n$-sources PID problem has conventionally been approached by defining a measure of redundant information which in turn implies a unique solution for the atoms of information. We showed how such a measure can be defined in terms of the information provided by logical statements of a specific form. We discussed furthermore how the parthood perspective can be utilized to systematically address the question of whether a PID may be determined based on concepts other than redundancy. In doing so, we showed that this is indeed possible in terms of measures of ``restricted information'', ``weak synergy'', and ``moderate synergy'' but not in terms of ``strong synergy''. We hope to have shown that there are deep connections between mereology, formal logic and information decomposition that future research in these fields may benefit from. 

\section{Appendix}

\subsection{Minimally Consistent PID}\label{app:min_cons_pid}
\begin{definition}[Minimally consistent PID] \label{def:mcpid_sum}
Let $S_1,\ldots,S_n, T$ be jointly distributed random variables with joint distribution $\mathbb{P}_J$ and let $\mathcal{B}_n$ be the set of parthood distributions in the context of n source variables. A minimally consistent partial-information-decomposition of the mutual information provided by the sources $S_1,\ldots,S_n$ about the target $T$ is any function $\Pi_{\mathbb{P}_j}:\mathcal{B}_n \rightarrow \mathbb{R}$, determined by $\mathbb{P}_J$, that satisfies
	\begin{equation}\label{eq:consistency_equation}
	I_{\mathbb{P}_J}(T:(S_i)_{i \in \mathbf{a}}) = \sum\limits_{f(\mathbf{a}) = 1} \Pi_{\mathbb{P}_J}(f)
	\end{equation}
	for all $\mathbf{a} \subseteq \{1,\ldots,n\}$. The subscripts $\mathbb{P}_J$ indicate that both the mutual information and the information atoms are functions of the underlying joint distribution.
\end{definition}

\subsection{Proof of isomorphism between $(\mathcal{B},\sqsubseteq)$, $(\mathcal{L},\leftmodels)$ and $(\mathcal{A},\preceq)$}
\label{app:proof_isomorph}
First, recall that the relation $\models$ of logical implication is formally defined in terms of the notion of a \textit{valuation} \cite{smullyan1995first}. A valuation is an assignment of truth-values (0 for false and 1 for true) over the propositional variables $\phi_i$. So the set of all valuations $\mathcal{V}$ is given by the set of all mappings from $\{\phi_1,\ldots,\phi_n\}$ into $\{0,1\}$: 
\begin{equation}
\mathcal{V} := \{0,1\}^{\{\phi_1,\ldots,\phi_n\}}
\end{equation}
A valuation is said to \textit{satisfy} a statement $\tilde{\alpha}$, written as $\models_V \tilde{\alpha}$, under the following conditions
\begin{enumerate}
	\item If $\tilde{\alpha}$ is an atomic statement, then ${\models_V \tilde{\alpha} \iff V(\tilde{\alpha})  = 1}$
	\item If $\tilde{\alpha}$ is of the form $\tilde{\beta} \wedge \tilde{\gamma}$, then ${\models_V \tilde{\alpha} \iff \models_V \tilde{\beta} \text{ and } \models_V \tilde{\gamma}}$
	\item If $\tilde{\alpha}$ is of the form $\tilde{\beta} \vee \tilde{\gamma}$, then ${\models_V \tilde{\alpha} \iff \models_V \tilde{\beta} \text{ or } \models_V \tilde{\gamma}}$
\end{enumerate}
In this way, the satisfaction relationship is inductively defined for all statements of the propositional language we are considering here. The relation of logical implication is now defined such that a statement $\tilde{\alpha}$ implies a statement $\tilde{\beta}$ just in case all valuations that satisfy $\tilde{\alpha}$ also satisfy $\tilde{\beta}$. Formally,
\begin{equation}
\tilde{\alpha} \models \tilde{\beta} \iff \forall V\in \mathcal{V}: \models_V \tilde{\alpha} \rightarrow \models_V \tilde{\beta}
\end{equation}

\begin{proof}[Proof of the theorem]
We first show the isomorphism between $(\mathcal{B},\sqsubseteq)$ and $(\mathcal{A},\preceq)$ and then the isomorphism between $(\mathcal{A},\preceq)$ and $(\mathcal{L},\leftmodels)$. The following mapping $\varphi: \mathcal{A} \rightarrow \mathcal{B}$ is an isomorphism between $(\mathcal{B},\sqsubseteq)$ and $(\mathcal{A},\preceq)$:
\begin{equation}
\varphi(\alpha) := f_\alpha \text{ with } f_\alpha(\mathbf{b}) = 
\begin{cases}
1 \text{ if } \exists \mathbf{a} \in \alpha: \mathbf{b} \supseteq \mathbf{a}  \\
0 \text{ otherwise}
\end{cases}
\end{equation}
First, $\varphi$ is surjective: let $f \in \mathcal{B}$, then $\varphi(\alpha_f) = f$ for the set $\alpha_f$ of minimal elements with value 1, i.e.
\begin{equation}
\alpha_f := \{\mathbf{a} \mid f(\mathbf{a}) = 1 \text{ \& } \neg \exists \mathbf{b} \subset \mathbf{a}: f(\mathbf{b})=1 \}
\end{equation}
$\varphi$ is also injective: let $\varphi(\alpha) = f_\alpha = f_\beta = \varphi(\beta)$ and let $\mathbf{b} \in \beta$. Then, $f_\beta(\mathbf{b}) = 1$ and hence $f_\alpha(\mathbf{b}) = 1$. Therefore, $\exists \mathbf{a} \in \alpha: \mathbf{b} \supseteq \mathbf{a}$. But this can only be true if $\mathbf{b} = \mathbf{a}$, because suppose $\mathbf{b} \supset \mathbf{a}$. We have $f_\beta(\mathbf{a})=1$ and hence $\exists \mathbf{b}^* \in \beta: \mathbf{a} \supseteq \mathbf{b}^*$. But then $\mathbf{b} \supset \mathbf{a} \supseteq \mathbf{b}^*$ while $\mathbf{b}, \mathbf{b}^* \in \beta$ contradicting the fact that $\beta$ is an antichain. Hence, $\mathbf{b} \in \alpha$. By the same argument it can be shown that any $\mathbf{a} \in \alpha$ has to be in $\beta$ and therefore $\alpha = \beta$.

It remains to be shown that $\varphi$ is structure preserving. So let $\alpha \preceq \beta$, i.e. $\forall \mathbf{b} \in \beta \exists \mathbf{a} \in \alpha: \mathbf{b} \supseteq \mathbf{a}$. We need to show that in this case $\varphi(\alpha) \sqsubseteq \varphi(\beta)$, i.e. $f_\beta(\mathbf{a}) = 1 \rightarrow f_\alpha(\mathbf{a})=1$. So let $f_\beta(\mathbf{a}) = 1$, then $\exists \mathbf{b} \in \beta: \mathbf{a} \supseteq \mathbf{b}$. By assumption this means that $\exists \mathbf{a}^* \in \alpha: \mathbf{b} \supseteq \mathbf{a}^*$. Hence $\mathbf{a} \supseteq \mathbf{a}^*$ and therefore $f_\alpha(\mathbf{a})=1$. Regarding the other direction suppose that $f \sqsubseteq g$. Now let $\mathbf{b} \in \beta_g = \varphi^{-1}(g)$, then $g(\mathbf{b}) = 1$ and hence $f(\mathbf{b}) = 1$. Therefore,  $\exists \mathbf{a} \in \alpha_f = \varphi^{-1}(f): \mathbf{b} \supseteq \mathbf{a}$, and thus, $\alpha_f \preceq \beta_g$.

We now turn to the isomorphism between $(\mathcal{L},\leftmodels)$ and $(\mathcal{A},\preceq)$. The mapping $\Psi : \mathcal{A} \rightarrow \mathcal{L}$ defined in the main text is an isomorphism. $\Psi$ is injective for let $\alpha, \beta \in \mathcal{A}$ be two distinct antichains. Then there has to be an $\mathbf{a} \in \alpha$ not contained in $\beta$ (or vice versa). But then the conjunction $\bigwedge_{i\in \mathbf{a}} \phi_i$ will appear in $\tilde{\alpha}$ while it does not appear in $\tilde{\beta}$. Accordingly, $\tilde{\alpha}$ and $\tilde{\beta}$ are distinct elements of $\mathcal{L}$. $\Psi$ is surjective as well for let $\tilde{\alpha} \in \mathcal{L}$. Then $\tilde{\alpha}$ is of the form $\bigvee_{\mathbf{j} \in \mathbf{J}} \bigwedge_{i\in \mathbf{j}} \phi_i$ for some set of index sets $\mathbf{J} = \{\mathbf{j}_1,\ldots,\mathbf{j}_m\}$ where $\mathbf{j}_i \subseteq \{1,\ldots,n\}$. Because the conjunctions $\bigwedge_{i\in \mathbf{j}} \phi_i$ have to be logically independent it follows that the index sets cannot be subsets of each other, i.e. $\neg (\mathbf{j}_k \supseteq \mathbf{j}_l)$ for $k\neq l $. But this implies that $\mathbf{J}$ is an antichain which is, by definition of $\Psi$, mapped onto $\tilde{\alpha}$. 
	
	It only remains to be shown that $\beta \preceq \alpha \iff \tilde{\beta} \leftmodels \tilde{\alpha}$. First, suppose that $\beta \preceq \alpha$. We need to show that for all valuations $V \in \mathcal{V} = \{0,1\}^{\{\phi_1,\ldots,\phi_n\}}$: $\models_V \tilde{\alpha} \rightarrow \models_V \tilde{\beta}$, i.e. all Boolean valuations of the $\phi_i$ that make $\tilde{\alpha}$ true, also make $\tilde{\beta}$ true. So suppose $\models_V \tilde{\alpha}$, then there must be an $a\in \alpha$ such that $\models_V \bigwedge_{i\in a} \phi_i$. But since $\beta \preceq \alpha$, there must be a $\mathbf{b}\in \beta$ such that $a \supseteq b$. Therefore, $\models_V \bigwedge_{i\in b} \phi_i$. Hence, V also satisfies the disjunction over all $\mathbf{b} \in \beta$: $\models_V \bigvee_{\mathbf{b}\in \beta}\bigwedge_{i\in b} \phi_i = \tilde{\beta}$. 
	
	Regarding the other direction, suppose that $\tilde{\beta} \leftmodels \tilde{\alpha}$, i.e. all valutions satisfying $\tilde{\alpha}$ also satisfy $\tilde{\beta}$. Now suppose for contradiction that $\neg (\beta \preceq \alpha)$, i.e. $\exists \mathbf{a}^* \in \alpha \forall \mathbf{b} \in \beta: \neg(\mathbf{a} \supseteq \mathbf{b})$. In this case, we can construct a valuation V that satisfies $\tilde{\alpha}$ but not $\tilde{\beta}$ in the following way:
	\begin{equation}
	V(\phi_i) =
	\begin{cases}
	1 &\text{if } i \in \mathbf{a}^* \\
	0 &\text{if } i \notin \mathbf{a}^*
	\end{cases}
	\end{equation}
	By construction all $b\in \beta$ contain at least one index i not contained in a. Therefore, V does not satisfy any of the conjunctions $\bigwedge_{i\in b} \phi_i$, and thus it does not satisfy $\tilde{\beta}$, in contradiction to the initial assumption. Hence, $\beta\preceq \alpha$, concluding the proof.
\end{proof}
\begin{cor}
$(\mathcal{L},\leftmodels)$ and $(\mathcal{B}, \sqsubseteq)$ are lattices.
\end{cor}
\begin{proof}
Follows from the isomorphism and the fact that $(\mathcal{A},\preceq)$ is a lattice as shown by \cite{Crampton2000}.
\end{proof}

\subsection{Proofs of Propositions }\label{app:proofs_of_props}
\subsubsection{Monotonic truth tables}
\begin{proof}[Proof of Proposition 1]
Let $\tilde{\alpha} \in \mathcal{L}$ and let $V,V^\prime \in \mathcal{V}$ such that $\forall i \in \{1,\ldots,n\}: V(\phi_i) = 1 \rightarrow V^\prime(\phi_i) = 1$. Suppose that $T_{\tilde{\alpha}}(V) = 1$. Then V must satisfy at least one of the conjunctions $\bigwedge_{i\in\mathbf{a}} \phi_i$. But since $V(\phi_i) = 1 \rightarrow V^\prime(\phi_i) = 1$ any conjunction satisfied by V must also be satisfied by $V^\prime$. Hence, $T_{\tilde{\alpha}}(V^\prime) = 1$. 
	
	Regarding the converse: let T be a monotonic truth-table. Then $T = T_{\tilde{\alpha}^*}$ for the statement
	\begin{equation}
	\tilde{\alpha}^* = \bigvee_{\substack{V \in \mathcal{V} \\ T(V)=1}} \bigwedge_{\substack{i \in \{1,\ldots,n\} \\ V(\phi_i) = 1}} \phi_i
	\end{equation} 
	Note that $\tilde{\alpha}^* $ is generally not in $\mathcal{L}$ because the conjunctions are not necessarily logically independent. But one can obtain an equivalent statement $\tilde{\alpha} \in \mathcal{L}$ by removing all conjunctions from $\tilde{\alpha}^*$ that logically imply another conjunction in $\tilde{\alpha}^*$. Let $\tilde{\alpha}$ be this statement. Then, if $\tilde{\alpha}$ is true, certainly $\tilde{\alpha}^*$ is true because the latter differs from the former only through additional disjuncts. Conversely, if $\tilde{\alpha}^*$ is true, then one of its conjuncts must be true. If the true conjunct in $\tilde{\alpha}^*$ does appear in $\tilde{\alpha}$ as well (i.e. it has not been removed), then trivially $\tilde{\alpha}$ has to be true as well. On the other hand, if this conjunct does not appear in $\tilde{\alpha}$, then it must have been removed which implies that there is a logically weaker conjunct in $\tilde{\alpha}$. But then this logically weaker conjunct has to be true as well, thereby making $\tilde{\alpha}$ true.  Therefore, $\tilde{\alpha}^*$ and $\tilde{\alpha}$ have the same truth-table T and $\tilde{\alpha} \in \mathcal{L}$ as desired. Furthermore, $\tilde{\alpha}$ is unique because $\models$ is antisymmetric on $\mathcal{L}$ by Corollary 1. Hence, there can be no two distinct but logically equivalent elements (i.e. elements with the same truth-table) in $\mathcal{L}$.
\end{proof}

\subsubsection{Characterization of Children}
\begin{proof}[Proof of Proposition 2]
Concerning the if-part we show the contraposition: Suppose that there is a $\tilde{\beta}$ strictly in between $\tilde{\gamma}$ and $\tilde{\alpha}$. If this is the case, then there must be a model $V_1$ such that $T_{\tilde{\beta}}(V_1) = 1$ while $T_{\tilde{\alpha}}(V_1) = 0$ and a distinct model $V_2$ such that $T_{\tilde{\gamma}}(V_2) = 1$ while $T_{\tilde{\beta}}(V_2) = 0$. But for both of these models it would be true that $T_{\tilde{\gamma}}(V_1) = 1$ while $T_{\tilde{\alpha}}(V_1) = 0$. Thus, $\tilde{\gamma}$ would be true in at least two additional cases.
	
	Concerning the only-if part we show the contraposition again: Suppose that $\tilde{\gamma}$ is true in the $k\geq 2$ additional cases contained in $ \mathcal{V}_* = \{V_1, V_2,\ldots,V_k\}$.  Consider the subset of these models with the smallest number of ones:
	\begin{equation}
	\mathcal{V}_*^{min} = \left\{V \in \mathcal{V}_* \hspace{0.2cm}|\hspace{0.2cm} \forall V^\prime \in \mathcal{V}_*: \sum_{i=1}^{n} V(\phi_i) \leq \sum_{i=1}^{n} V^\prime(\phi_i)   \right\}
	\end{equation}
	Now let $V_* \in \mathcal{V}_*^{min}$. Then the truth table
	\begin{equation}
	T_{\tilde{\beta}}(V) := 
	\begin{cases}
	1 & \text{if } T_{\tilde{\gamma}}(V) = 1 \text{ but } V \neq V_*\\
	0 & \text{otherwise}
	\end{cases}
	\end{equation}
	is monotonic and the statement $\tilde{\beta}$ associated with this truth-table is strictly in between $\tilde{\gamma}$ and $\tilde{\alpha}$. The latter is true because all valuations that satisfy $\tilde{\alpha}$ also satisfy $\tilde{\beta}$ and all valuations that satisfy $\tilde{\beta}$ also satisfy $\tilde{\gamma}$. At the same time there is a valuation, namely $V_*$, that satisfies $\tilde{\gamma}$ but not $\tilde{\beta}$, and a set of valuations with at least one element, namely $\mathcal{V}_* \backslash \{V_*\}$, that satisfies $\tilde{\beta}$ but not $\tilde{\alpha}$. Thus, all three statements have to be distinct. Regarding the monotonicity: by assumption $\tilde{\gamma}$ has a monotonic truth-table and the truth-table of $\tilde{\beta}$  is identical except that $T_{\tilde{\beta}}(V_*) = 0$. So the only way $T_{\tilde{\beta}}$ could \textit{not} be monotonic would be for there to exist a valuation $V_*^\prime$, distinct from $V_*$, that would enforce $T_{\tilde{\beta}}(V_*) = 1$ via monotonicity, i.e. a valuation that results from flipping some ones in $V_*$ to zeros and that satisfies $\tilde{\beta}$. Suppose there is such a valuation. $V_*^\prime$ would have to satisfy $\tilde{\beta}$ while not satisfying $\tilde{\alpha}$, since if it did satisfy $\tilde{\alpha}$, $V_*$ would have to satisfy  $\tilde{\alpha}$ as well in contradiction to $V_* \in \mathcal{V}_*$. Furthermore, as $V_*^\prime$ satisfies $\tilde{\beta}$ it also satisfies $\tilde{\gamma}$. Therefore,  $V_*^\prime \in \mathcal{V}_*$. However, if it were true that $V_*^\prime(\phi_i) = 1\rightarrow V_*(\phi_i) = 1$, then $\sum_{i=1}^{n} V_*^\prime(\phi_i) < \sum_{i=1}^{n} V_*(\phi_i)$, contradicting the fact that $V_* \in \mathcal{V}_*^{min}.$ 
\end{proof}
\subsubsection{Lower bound on children}
\begin{proof}[Proof of Proposition 3]
	Let $\alpha$ be such an antichain and let $\mathbf{a} \in \alpha$ be a set of indices such that $|\mathbf{a}|= k $. We utilize the isomorphism between $\mathcal{A}$ and $\mathcal{L}$ by showing that $\tilde{\alpha}$ has at least k children. 
	Since $|\mathbf{a}| = k$ there are exactly k distinct indices $i_1,\ldots,i_k \in \mathbf{a}$ and we can define k subsets of valuations
	\begin{align}
	\mathcal{V}_1 &= \{V\in \mathcal{V}: \neg(\models_V \tilde{\alpha}) \text{ \& } i \in \mathbf{a} \backslash \{i_1\} \rightarrow V(\phi_i) = 1  \} \\
	&\dots \notag \\
	\mathcal{V}_k &= \{V\in \mathcal{V}: \neg(\models_V \tilde{\alpha}) \text{ \& } i \in \mathbf{a} \backslash \{i_k\} \rightarrow V(\phi_i) = 1  \}
	\end{align} 
	In other words, the valuations in $\mathcal{V}_1$, first, do not satisfy $\tilde{\alpha}$, and second,  assign a one to all $\phi_i$ if $i$ is in the collection $\mathbf{a}$ but not equal to $i_1$. The definition of the other $\mathcal{V}_i$ is analogous. The goal is now to find 'maximal' valuations (making as many $\phi_i$ true as possible) in these sets and modify the truth-table of $\tilde{\alpha}$ by assigning a one to exactly one of these valuations. This can be done for all valuations separately to obtain k novel monotonic truth-tables. These monotonic truth-tables are uniquely associated with specific statements via Proposition 1 which can then be shown to be children by Proposition 2 since they are true in exactly one more case than $\tilde{\alpha}$. 
	To make this argument note first that $\mathcal{V}_1, \ldots,\mathcal{V}_k$ each contain at least one element $V_1, \ldots,V_k$ respectively:
	\begin{align}
	V_1(\phi_i) &=
	\begin{cases}
	1 & \text{if } i \in \mathbf{a} \backslash  \{i_1\} \\
	0 & \text{otherwise}
	\end{cases} \\
	&\dots \notag\\
	V_k(\phi_i) &=
	\begin{cases}
	1 & \text{if } i \in \mathbf{a} \backslash  \{i_k\} \\
	0 & \text{otherwise}
	\end{cases}
	\end{align}
	These valuations do not satisfy $\tilde{\alpha}$: They don't satisfy the conjunction $\bigwedge_{i\in \mathbf{a}} \phi_i$ and since $\alpha$ is an antichain each  $\mathbf{a}^\prime \neq \mathbf{a}$  has to contain at least one index j not contained in $\mathbf{a}$. The corresponding conjunctions $\bigwedge_{i\in \mathbf{a}^\prime} \phi_i = \phi_j \wedge \bigwedge_{i\in \mathbf{a}^\prime \backslash \{j\}} \phi_i$ are therefore not satisfied by any $V_i$ since by construction $V_1(\phi_j) = \ldots = V_k(\phi_j) = 0$. Now consider the sets of 'maximal' valuations within the $\mathcal{V}_i$:
	\begin{align}
	\mathcal{V}_1^{max} &= \left\{ V\in \mathcal{V}_1|  \forall V^\prime \in \mathcal{V}_1: \sum_{i=1}^n V^\prime(\phi_i) \leq \sum_{i=1}^n V(\phi_i)  \right\} \\
	&\dots \notag \\
	\mathcal{V}_k^{max} &= \left\{V\in \mathcal{V}_k|  \forall V^\prime \in \mathcal{V}_k: \sum_{i=1}^n V^\prime(\phi_i) \leq \sum_{i=1}^n V(\phi_i) \right\}
	\end{align} 
	Let $V_1^* \in \mathcal{V}_1^{max}, ..., V_k^* \in \mathcal{V}_k^{max} $. Due to the maximality of these valuations the following truth-tables are monotonic
	\begin{align}
	T_{\tilde{\gamma}_1}(V) &= 
	\begin{cases}
	1 & \text{if } T_{\tilde{\alpha}}(V) = 1 \text{ or } V = V_1^* \\
	0 & \text{otherwise}
	\end{cases} 
	\\
	&\dots \notag\\
	T_{\tilde{\gamma}_k}(V) &= 
	\begin{cases}
	1 & \text{if } T_{\tilde{\alpha}}(V) = 1 \text{ or } V = V_k^* \\
	0 & \text{otherwise}
	\end{cases}
	\end{align}
	This is because, first, the truth-table of $\tilde{\alpha}$ is already monotonic, and second, if a zero is flipped to a one in $V_1^*$ or $\dots$ or $M_k^*$ the resulting valuations are by construction guaranteed to satisfy $\tilde{\alpha}$. Otherwise, we would obtain valuations in $\mathcal{V}_1$ or $\dots$ or $\mathcal{V}_k$ respectively, containing more ones than $V_1^*$ or $\dots$ or $V_k^*$ respectively, in contradiction to the maximality of these valuations. The uniquely defined statements $\tilde{\gamma}_1,\ldots, \tilde{\gamma}_k$ corresponding to these truth-tables via Proposition 1 are children of $\tilde{\alpha}$ by Proposition 2 because each of them is true in exactly one additional valuation compared to $\tilde{\alpha}$. Finally all of these statements are distinct since they are pairwise logically independent and a single statement cannot have multiple truth-tables.
\end{proof}
\subsubsection{Algorithm to determine children}
\begin{proof}[Proof of Proposition 4]
	Firstly, any $\tilde{\gamma}$ produced by the algorithm is a direct child since its truth-table differs from that of $\tilde{\alpha}$ only through an additional one, i.e. $\tilde{\gamma}$ is true in exactly one more case than $\tilde{\alpha}$ and is thus a direct child by Proposition 2. Secondly, there is no child of $\tilde{\alpha}$ that is not generated by the algorithm. Again by Proposition 2, the truth-table of any such child would differ from that of $\tilde{\alpha}$ only through a single one. But the algorithm explores systematically \textit{all} possibilities to add a single one to the truth-table of $\tilde{\alpha}$. Thus any child $\tilde{\gamma}$ will be generated at some point.
\end{proof}
A pseudocode version of the algorithm is shown in Algorithm \ref{algo:children}.

\begin{algorithm}\label{algo:children}
\SetKwInOut{Input}{inputs}
\SetKwInOut{Output}{outputs}
\SetKwProg{GetChld}{GetChld}{}{}
\GetChld{$\tilde{\alpha}$}{
    \Input{A statement $\tilde{\alpha}$}
    \Output{The set of children of $\tilde{\alpha}$ denoted by $\mathcal{C}_{\tilde{\alpha}}$}
    $k\gets 0$\;
    $\mathcal V_{\tilde{\alpha}} \gets \emptyset$\;
    $\mathcal{C}_{\tilde{\alpha}}\gets\emptyset$\;
    \tcp{step (1)}
    \ForEach{valuation $V\in\mathcal{V}$}{%
        \If{$\not\models_V \tilde{\alpha}$}{
            $\mathcal V_{\tilde{\alpha}}\gets \mathcal V_{\tilde{\alpha}}\cup V$\;
            \tcp{Maximal number of ones in $V$ if $\not\models_V \tilde{\alpha}$}
            \If{$\sum_i V_i>k$}{
            $k\gets \sum_i V_i$\;
            }
        }
    }
    \tcp{step (3) as a while loop}
    \While{$k\neq 0$}{
        \tcp{Construct the set of all $V\in\mathcal V_{\tilde{\alpha}}$ such that $\sum_i V_i=k$}
        $\mathcal V_{\tilde{\alpha}}^k \gets \emptyset$\;
        \ForEach{valuation $V\in\mathcal{V}_{\tilde\alpha}$}{%
            \If{$\sum_i V_i = k$}{
               $\mathcal V_{\tilde{\alpha}}^k \gets V$\; 
            }
        }
        \tcp{Construct a child of $\tilde{\alpha}$ if it exists (step (2))}
        \ForEach{valuation $V\in\mathcal V_{\tilde{\alpha}}^k$}{%
            $Q \gets \emptyset$\;
            \For{$V'\in\mathcal V_{\tilde{\alpha}}$}{%
                \If{$\sum_i V'_i = k+1$ and $V(\phi_i) = 1\rightarrow V'(\phi_i) = 1~\forall i \in [n]$}{%
                $Q \gets V'$\;
                break\;
                }
            }
            \If{$Q=\emptyset$}{%
                    construct $\tilde\gamma$ that satisfies $V$ and every $V'\in\mathcal V\backslash\mathcal V_{\tilde{\alpha}}$\; 
                    $\mathcal{C}_{\tilde{\alpha}}\gets\mathcal{C}_{\tilde{\alpha}}\cup\tilde\beta$
                }
        }
        $k \gets k - 1$
    }
    \Return{$\mathcal{C}_{\tilde{\alpha}}$}
 }
\caption{Determines children of a statement $\tilde{\alpha}$ in the logic lattice.}
\end{algorithm}

\subsubsection{Meet and Join operations on logic lattices}
The meet $\tilde{\wedge}$ and join $\tilde{\vee}$ operations can be explicitly constructed in the following way: The element of $\mathcal{L}$ logically equivalent to the disjunction $\ta \vee \tb$ can be obtained by simply removing all disjuncts that logically imply another disjunct. The element of $\mathcal{L}$ logically equivalent to the conjunction $\ta \wedge \tb$ can be obtained by, first, applying the distributive law to obtain a disjunction of conjunctions, second, applying the idempotency law to all conjunctions to remove repeated statements, and third, removing again all disjuncts that logically imply another disjunct. Denoting these three operations by $\mathcal{D}$, $\mathcal{I}$, and $\underline{\hspace{1mm}\circ\hspace{1mm}}$ (underline) respectively, the meet and join have the explicit expressions given in the following proposition:
\begin{prop}[Meet and Join Operations]
\begin{align}
\ta \meet \tb &= \underline{\ta \vee \tb} \\
\ta \join \tb &= \underline{\mathcal{I}(\mathcal{D}(\ta \wedge \tb))}
\end{align}
\end{prop}
\begin{proof}
By construction, $\underline{\ta \vee \tb}$ and $\underline{\mathcal{I}(\mathcal{D}(\ta \wedge \tb))}$ are in $\mathcal{L}$. Furthermore, since the operations $\mathcal{D}$, $\mathcal{I}$, and $\underline{\circ}$ do not affect the truth-conditions of statements,  $\underline{\ta \vee \tb}$ and $\underline{\mathcal{I}(\mathcal{D}(\ta \wedge \tb))}$ are logically equivalent to $\ta \vee \tb$ and $\ta \wedge \tb$, respectively. Hence, it only needs to be shown that these latter statements satisfy the conditions of meet and join respectively. Now, clearly $\ta \vee \tb$ is logically weaker than both $\ta$ and $\tb$ while $\ta \wedge \tb$ is logically stronger than both $\ta$ and $\tb$. It remains to be shown that former is the strongest such statement while the latter is the weakest such statement. Suppose there was statement $\tilde{\gamma}$ stronger than $\ta \vee \tb$, then there would have to be a model $M^*$ making $\tc$ false and  $\underline{\ta \vee \tb}$ true. But since $\underline{\ta \vee \tb}$ is true whenever either $\ta$ is true or $\tb$ is true, this means that $\tc$ would have to be false in a case where one of $\ta$ or $\tb$ is true. However, this implies that $\tc$ cannot be logically weaker than both $\ta$ and $\tb$, and hence, $\underline{\ta \vee \tb}$  must be the strongest statement logically weaker than $\ta$ and $\tb$. Now suppose there was a statement $\tilde{\gamma}$ weaker than $\ta \wedge \tb$, then there would have to be a model $M^*$ making $\tc$ true but $\ta \wedge \tb$ false. But this means that $\tc$ would be true in a case in which either $\ta$ or $\tb$ is false. Accordingly, $\tc$ cannot be stronger than both $\ta$ and $\tb$, and hence, $\underline{\mathcal{I}(\mathcal{D}(\ta \wedge \tb))}$ must be the weakest statement logically stronger than $\ta$ and $\tb$.
\end{proof}

\subsection{Derivations related to restricted information based and synergy based PID} \label{app:restr_info}

\subsubsection{Relation between restricted information and conditional mutual information}

The relation between restricted information and conditional mutual information given by Equation 5.5 can be derived via the chain rule as follows:
{\scriptsize
\begin{align}
I\left(T:(S_i)_{i\in \alpha_\cup} | (S_j)_{j \in \alpha_\cup^C}\right) \!&=\! I(T :(S_i)_{i\in [n]}) - I(T:(S_j)_{j \in \alpha_\cup^C})  \\
&= \sum\limits_{f([n]) = 1} \Pi(f) - \sum\limits_{f(\alpha_\cup^C) = 1} \Pi(f) \\
&= \sum\limits_{f(\alpha_\cup^C) = 0} \Pi(f) \\
&= \sum\limits_{f(\mathbf{b}) = 1 \rightarrow \exists j: \{i_j\} \supseteq \mathbf{b} }  \Pi(f) \\
&= I_\text{res}(T:\alpha)
\end{align}
}

\subsubsection{Proof that moderate synergy induces a unique PID}\label{app:moderate_synergy_proof}
The claim that defining a measure of moderate synergy leads to a unique solution for the atoms of information can be shown by starting from the system of equation associated with weak synergy. These equations can be transformed into the moderate synergy equations by operations that preserve invertibility. First, the ``self-synergy'' equations
\begin{equation}
I_\text{ws/ms}(T:\bfa) = I(T:\bfa^C|\bfa) = \sum\limits_{f(\bfa)=0} \Pi(f)
\end{equation}
are contained in both systems. Furthermore, weak and moderate synergy coincide if $\alpha_\cup = [n]$. In this case, the additional constraint $f(\alpha_\cup) = 1$ is superfluous since $f([n])$ is necessarily equal to 1 by the properties of parthood distributions. Thus, the corresponding equations are again contained in both systems. This only leaves the case of $\alpha_\cup \subset [n]$ while $|\alpha| \geq 2$. Let $\alpha$ be such an antichain. It can be shown that the corresponding moderate synergies can be expressed as a difference between two equations in the weak synergy system:
\begin{align}
I_\text{ws}(T:\alpha) -I_\text{ms}(T:\alpha) &= \sum_{\substack{\forall \bfa_i: f(\bfa_i)=0 \\ f(\alpha_\cup)=0}} \Pi(f) \\
&= \sum_{f(\alpha_\cup)= 0} \Pi(f) \\
&= I(T:\alpha_\cup^C|\alpha_\cup)
\end{align}
where the second to last equality follows because the monotonicity of parthood distributions implies that $f(\alpha_\cup) = 0 \rightarrow f(\mathbf{a}) = 0~\forall \bfa \in \alpha$. Therefore, we obtain
\begin{align}
I_\text{ms}(T:\alpha) &= I_\text{ws}(T:\alpha) - I(T:\alpha_\cup^C|\alpha_\cup) \\
&= I_\text{ws}(T:\alpha) - I_\text{ws}(T:\alpha_\cup)
\end{align}
showing that the moderate synergy equation associated with $\alpha$ is the difference between two weak synergy equations. Since subtracting two equations from each other leaves invertibility unaffected this establishes that the moderate synergy system of equations is invertible as well. 
 
\vskip6pt

\enlargethispage{20pt}



\aucontribute{AG conceived the parthood-based and logic-based formulations of PID and wrote the original manuscript except for the introduction which was provided by MW. MW originally conceived the $i_\cap^{sx}$ measure of redundant information rederived in \S\ref{sec:logic_isx}. AM provided critical feedback regarding the mathematical aspects of the paper. All authors were involved in revising the manuscript and refining the ideas presented therein. All authors gave final approval for publication and agree to be held accountable for the work performed therein.}

\competing{We declare we have no competing interests.}

\funding{MW, AM, AG are employed at the Campus Institute for Dynamics of Biological Networks (CIDBN) funded by the Volkswagen Stiftung. MW, AM received support from the Volkswagenstiftung under the programme 'Big Data in den Lebenswissenschaften'. This work was supported by a funding from the Ministry for Science and Education of Lower Saxony and the Volkswagen Foundation through the ``Nieders\"{a}chsisches Vorab''.  }

\ack{We thank Kyle Schick-Poland, David Ehrlich, and Andreas Schneider for helpful comments on the draft.}


\appendix


\end{document}